\documentclass[english]{article}
\usepackage[T1]{fontenc}
\usepackage[latin9]{inputenc}
\usepackage{units}
\usepackage{amsmath}
\usepackage{amsthm}
\usepackage{amssymb}
\usepackage{graphicx}

\usepackage[numbers]{natbib}

\makeatletter
\theoremstyle{plain}
\newtheorem{thm}{\protect\theoremname}
\theoremstyle{remark}
\newtheorem{rem}[thm]{\protect\remarkname}
\theoremstyle{plain}
\newtheorem{prop}[thm]{\protect\propositionname}
\theoremstyle{plain}
\newtheorem{cor}[thm]{\protect\corollaryname}
\theoremstyle{definition}
\newtheorem{example}[thm]{\protect\examplename}

\usepackage[final]{neurips_2023}

\usepackage{hyperref}       
\usepackage{url}            

\author{%
  Arthur Jacot \\
  Courant Institute of Mathematical Sciences\\
  New York University\\
  New York, NY 10012 \\
  \texttt{arthur.jacot@nyu.edu}
}

\@ifundefined{showcaptionsetup}{}{%
 \PassOptionsToPackage{caption=false}{subfig}}
\usepackage{subfig}
\makeatother

\usepackage{babel}
\providecommand{\corollaryname}{Corollary}
\providecommand{\examplename}{Example}
\providecommand{\propositionname}{Proposition}
\providecommand{\remarkname}{Remark}
\providecommand{\theoremname}{Theorem}

\begin{document}
\title{Bottleneck Structure in Learned Features:\\
Low-Dimension vs Regularity Tradeoff}
\maketitle
\begin{abstract}
Previous work \cite{jacot_2022_BN_rank} has shown that DNNs with
large depth $L$ and $L_{2}$-regularization are biased towards learning
low-dimensional representations of the inputs, which can be interpreted
as minimizing a notion of rank $R^{(0)}(f)$ of the learned function
$f$, conjectured to be the Bottleneck rank. We compute finite depth
corrections to this result, revealing a measure $R^{(1)}$ of regularity
which bounds the pseudo-determinant of the Jacobian $\left|Jf(x)\right|_{+}$
and is subadditive under composition and addition. This formalizes
a balance between learning low-dimensional representations and minimizing
complexity/irregularity in the feature maps, allowing the network
to learn the `right' inner dimension. Finally, we prove the conjectured
bottleneck structure in the learned features as $L\to\infty$: for
large depths, almost all hidden representations are approximately
$R^{(0)}(f)$-dimensional, and almost all weight matrices $W_{\ell}$
have $R^{(0)}(f)$ singular values close to 1 while the others are
$O(L^{-\frac{1}{2}})$. Interestingly, the use of large learning rates
is required to guarantee an order $O(L)$ NTK which in turns guarantees
infinite depth convergence of the representations of almost all layers.
\end{abstract}

\section{Introduction}

The representation cost $R(f;L)=\min_{\theta:f_{\theta}=f}\left\Vert \theta\right\Vert ^{2}$
\cite{dai_2021_repres_cost_DLN} can be defined for any model and
describes the bias in function space resulting from the minimization
of the $L_{2}$-norm of the parameters. While it can be computed explicitly
for linear networks \cite{dai_2021_repres_cost_DLN} or shallow nonlinear
ones \cite{bach2017_F1_norm}, the deep non-linear case remains ill-understood
\cite{jacot_2022_L2_reformulation}.

Previous work \cite{jacot_2022_BN_rank} has shown that the representation
cost of DNNs with homogeneous nonlinearity $\sigma$ converges to
a notion of rank over nonlinear functions$\lim_{L\to\infty}\frac{R(f;L)}{L}\to R^{(0)}(f)$.
Over a large set of functions $f$, the limiting representation cost
$R^{(0)}(f)$ was proven the so-called Bottleneck (BN) rank $\mathrm{Rank}_{BN}(f)$
which is the smallest integer $k$ such that $f:\mathbb{R}^{d_{in}}\to\mathbb{R}^{d_{out}}$
can be factored $f=\mathbb{R}^{d_{in}}\stackrel{g}{\longrightarrow}\mathbb{R}^{k}\stackrel{h}{\longrightarrow}\mathbb{R}^{d_{out}}$
with inner dimension $k$ (it is conjectured to match it everywhere).
This suggests that large depth $L_{2}$-regularized DNNs are adapted
for learning functions of the form $f^{*}=g\circ h$ with small inner
dimension.

This can also be interpreted as DNNs learning symmetries, since a
function $f:\Omega\to\mathbb{R}^{d_{out}}$ with symmetry group $G$
(i.e. $f(g\cdot x)=f(x)$) can be defined as mapping the inputs $\Omega$
to an embedding of the modulo space $\nicefrac{\Omega}{G}$ and then
to the outputs $\mathbb{R}^{d_{out}}$. Thus a function with a lot
of symmetries will have a small BN-rank, since $\mathrm{Rank}_{BN}(f)\leq\dim\left(\nicefrac{\Omega}{G}\right)$
where $\dim\left(\nicefrac{\Omega}{G}\right)$ is the smallest dimension
$\nicefrac{\Omega}{G}$ can be embedded into.

A problem is that this notion of rank does not control the regularity
of $f$, but results of \cite{jacot_2022_BN_rank} suggest that a
measure of regularity might be recovered by studying finite depths
corrections to the $R^{(0)}$ approximation. This formalizes the balance
between minimizing the dimension of the learned features and their
complexity.

Another problem is that minimizing the rank $R^{(0)}$ does not uniquely
describe the learned function, as there are many fitting functions
with the same rank. Corrections allow us to identify the learned function
amongst those.

Finally, the theoretical results and numerical experiments of \cite{jacot_2022_BN_rank}
strongly suggest a bottleneck structure in the learned features for
large depths, where the (possibly) high dimensional input data is
mapped after a few layers to a low-dimensional hidden representation,
and keeps the approximately same dimensionality until mapping back
to the high dimensional outputs in the last few layers. We prove the
existence of such a structure, but with potentially multiple bottlenecks.

\subsection{Contributions}

We analyze the Taylor approximation of the representation cost around
infinite depth $L=\infty$:
\[
R(f;\Omega,L)=LR^{(0)}(f;\Omega)+R^{(1)}(f;\Omega)+\frac{1}{L}R^{(2)}(f;\Omega)+O(L^{-2}).
\]
The first correction $R^{(1)}$ measures some notion of regularity
of the function $f$ that behaves sub-additively under composition
$R^{(1)}(f\circ g)\leq R^{(1)}(f)+R^{(1)}(g)$ and under addition
$R^{(1)}(f+g)\leq R^{(1)}(f)+R^{(1)}(g)$ (under some constraints),
and controls the Jacobian of $f$: $\forall x,2\log\left|Jf(x)\right|_{+}\leq R^{(1)}(f)$,
where $\left|\cdot\right|_{+}$ is the \emph{pseudo-determinant},
the product of the non-zero singular values.

This formalizes the balance between the bias towards minimizing the
inner dimension described by $R^{(0)}$ and a regularity measure $R^{(1)}$.
As the depth $L$ grows, the low-rank bias $R^{(0)}$ dominates, but
even in the infinite depth limit the regularity $R^{(1)}$ remains
relevant since there are typically multiple fitting functions with
matching $R^{(0)}$ which can be differentiated by their $R^{(1)}$
value.

For linear networks, the second correction $R^{(2)}$ guarantees infinite
depth convergence of the representations of the network. We recover
several properties of $R^{(2)}$ in the nonlinear case, but we also
give a counterexample that shows that norm minimization does not guarantee
convergence of the representation, forcing us to look at other sources
of bias.

To solve this issue, we show that a $\Theta(L^{-1})$ learning rate
forces the NTK to be $O(L)$ which in turn guarantees the convergence
as $L\to\infty$ of the representations at almost every layer of the
network.

Finally we prove the Bottleneck structure that was only observed empricially
in \cite{jacot_2022_BN_rank}: we show that the weight matrices $W_{\ell}$
are approximately rank $R^{(0)}(f)$, more precisely $W_{\ell}$ has
$R^{(0)}(f)$ singular values that are $O(L^{-\frac{1}{2}})$ close
to $1$ and all the other are $O(L^{-\frac{1}{2}})$. Together with
the $O(L)$ NTK assumption, this implies that the pre-activations
$\alpha_{\ell}(X)$ of a general dataset at almost all layer is approximately
$R^{(0)}(f)$-dimensional, more precisely that the $k+1$-th singular
value of $\alpha_{\ell}(X)$ is $O(L^{-\frac{1}{2}})$.

\subsection{Related Works}

The representation cost has mostly been studied in settings where
an explicit formula can be obtained, such as in linear networks \cite{dai_2021_repres_cost_DLN},
or shallow nonlinear networks \cite{bach2017_F1_norm}, or for deep
networks with very specific structure \cite{ongie2022_linear_layer_in_DNN,le_2022_IB_mix_linear_homogeneous}.
A low rank phenomenon in large depth $L_{2}$-regularized DNNs has
been observed in \cite{timor_2022_implicit_large_depth}.

Regarding deep fully-connected networks, two reformulations of the
representation cost optimization have been given in \cite{jacot_2022_L2_reformulation},
which also shows that the representation becomes independent of the
width as long as the width is large enough.

The Bottleneck structure that we describe in this paper is similar
to the Information Bottleneck theory \cite{tishby2015_info_bottleneck}.
It is not unlikely that the bias towards dimension reduction in the
middle layers of the network could explain the loss of information
that was observed in the first layers of the network in \cite{tishby2015_info_bottleneck}.

\section{Setup}

In this paper, we study fully connected DNNs with $L+1$ layers numbered
from $0$ (input layer) to $L$ (output layer). The layer $\ell\in\{0,\dots,L\}$
has $n_{\ell}$ neurons, with $n_{0}=d_{in}$ the input dimension
and $n_{L}=d_{out}$ the output dimension. The pre-activations $\tilde{\alpha}_{\ell}(x)\in\mathbb{R}^{n_{\ell}}$
and activations $\alpha_{\ell}(x)\in\mathbb{R}^{n_{\ell}}$ are defined
by
\begin{align*}
\alpha_{0}(x) & =x\\
\tilde{\alpha}_{\ell}(x) & =W_{\ell}\alpha_{\ell-1}(x)+b_{\ell}\\
\alpha_{\ell}(x) & =\sigma\left(\tilde{\alpha}_{\ell}(x)\right),
\end{align*}
for the $n_{\ell}\times n_{\ell-1}$ connection weight matrix $W_{\ell}$,
the $n_{\ell}$-dim bias vector $b_{\ell}$ and the nonlinearity $\sigma:\mathbb{R}\to\mathbb{R}$
applied entry-wise to the vector $\tilde{\alpha}_{\ell}(x)$. The
parameters of the network are the collection of all connection weights
matrices and bias vectors $\theta=\left(W_{1},b_{1},\dots,W_{L},b_{L}\right)$.
The network function $f_{\mathbf{\theta}}:\mathbb{R}^{d_{in}}\to\mathbb{R}^{d_{out}}$
is the function that maps an input $x$ to the pre-activations of
the last layer $\tilde{\alpha}_{L}(x)$.

We assume that the nonlinearity is of the form $\sigma_{a}(x)=\begin{cases}
x & \text{if \ensuremath{x\geq0}}\\
ax & \text{otherwise}
\end{cases}$ for some $\alpha\in(-1,1)$ (yielding the ReLU for $\alpha=0$),
as any homogeneous nonlinearity $\sigma$ (that is not proportional
to the identity function, the constant zero function or the absolute
function) matches $\sigma_{a}$ up to scaling and inverting the inputs.

The functions that can be represented by networks with homogeneous
nonlinearities and any finite depth/width are exactly the set of finite
piecewise linear functions (FPLF) \cite{arora_2018_relu_piecewise_lin,he_2018_relu_piecewise_lin}.
\begin{rem}
In most of our results, we assume that the width is sufficiently large
so that the representation cost matches the infinite-width representation
cost. For a dataset of size $N$, a width of $N(N+1)$ suffices, as
shown in \cite{jacot_2022_L2_reformulation} (though a much smaller
width is often sufficient).
\end{rem}

\subsection{Representation Cost}

The representation cost $R(f;\Omega,L)$ is the minimal squared parameter
norm required to represent the function $f$ over the input domain
$\Omega$:
\[
R(f;\Omega,L)=\min_{\theta:f_{\mathbf{\theta}|\Omega}=f_{|\Omega}}\left\Vert \theta\right\Vert ^{2}
\]
where the minimum is taken over all weights $\theta$ of a depth $L$
network (with some finite widths $n_{1},\dots,n_{L-1}$) such that
$f_{\mathbf{\theta}}(x)=f(x)$ for all $x\in\Omega$. If no such weights
exist, we define $R(f;\Omega,L)=\infty$. 

The representation cost describes the natural bias on the represented
function $f_{\mathbf{\theta}}$ induced by adding $L_{2}$ regularization
on the weights $\theta$:
\[
\min_{\mathbf{\theta}}C(f_{\mathbf{\theta}})+\lambda\left\Vert \theta\right\Vert ^{2}=\min_{f}C(f)+\lambda R(f;\Omega,L)
\]
for any cost $C$ (defined on functions $f:\Omega\mapsto\mathbb{R}^{d_{out}}$)
and where the minimization on the right is over all functions $f$
that can be represented by a depth $L$ network with nonlinearity
$\sigma$.

For any two functions $f,g$, we denote $f\to g$ the function $h$
such that $g=h\circ f$, assuming it exists, and we write $R(f\to g;\Omega,L)=R(h;f(\Omega),L)$.
\begin{rem}
The representation cost also describes the implicit bias of networks
trained with the cross-entropy loss \cite{soudry2018implicit,gunasekar_2018_implicit_bias,chizat_2020_implicit_bias}.
\end{rem}

\section{Representation Cost Decomposition}

Since there are no explicit formula for the representation cost of
deep nonlinear networks, we propose to approximate it by a Taylor
decomposition in $\nicefrac{1}{L}$ around $L=\infty$. This is inspired
by the behavior of the representation cost of deep linear networks
(which represent a matrix as a product $A_{\theta}=W_{L}\cdots W_{1}$),
for which an explicit formula exists \cite{dai_2021_repres_cost_DLN}:
\[
R(A;L)=\min_{\theta:A=A_{\theta}}\left\Vert \theta\right\Vert ^{2}=L\left\Vert A\right\Vert _{\nicefrac{2}{L}}^{\nicefrac{2}{L}}=L\sum_{i=1}^{\mathrm{Rank}A}s_{i}(A)^{\nicefrac{2}{L}},
\]
where $\left\Vert \cdot\right\Vert _{p}^{p}$ is the $L_{p}$-Schatten
norm, the $L_{p}$ norm on the singular values $s_{i}(A)$ of $A$.

Approximating $s^{\frac{2}{L}}=1+\frac{2}{L}\log s+\frac{2}{L^{2}}(\log s)^{2}+O(L^{-3})$,
we obtain
\[
R(A;L)=L\mathrm{Rank}A+2\log\left|A\right|_{+}+\frac{1}{2L}\left\Vert \log_{+}A^{T}A\right\Vert ^{2}+O(L^{-2}),
\]
where $\log_{+}$ is the pseudo-log, which replaces the non-zero eigenvalues
of $A$ by their log.

We know that gradient descent will converge to parameters $\theta$
representing a matrix $A_{\theta}$ that locally minimize the loss
$C(A)+\lambda R(A;L)$. The approximation $R(A;L)\approx L\mathrm{Rank}A$
fails to recover the local minima of this loss, because the rank has
zero derivatives almost everywhere. But this problem is alleviated
with the second order approximation $R(A;L)\approx L\mathrm{Rank}A+2\log\left|A\right|_{+}$.
The minima can then be interpreted as first minimizing the rank, and
then choosing amongst same rank solutions the matrix with the smallest
pseudo-determinant. Changing the depth allows one to tune the balance
between minimizing the rank and the regularity of the matrix $A$.

\subsection{First Correction: Regularity Control\label{subsec:First-Correction}}

As a reminder, the dominating term in the representation cost $R^{(0)}(f)$
is conjectured in \cite{jacot_2022_BN_rank} to converge to the so-called
Bottleneck rank $\mathrm{Rank}_{BN}(f;\Omega)$ which is the smallest
integer $k$ such that $f$ can be decomposed as $f=\Omega\stackrel{g}{\longrightarrow}\mathbb{R}^{k}\stackrel{h}{\longrightarrow}\mathbb{R}^{d_{out}}$
with inner dimension $k$, and where $g$ and $h$ are FPLF. A number
of results supporting this conjecture are proven in \cite{jacot_2022_BN_rank}:
a sandwich bound
\[
\mathrm{Rank}_{J}(f;\Omega)\leq R^{(0)}(f;\Omega)\leq\mathrm{Rank}_{BN}(f;\Omega)
\]
for the Jacobian rank $\mathrm{Rank}_{J}(f;\Omega)=\max_{x}\mathrm{Rank}Jf(x)_{|T_{x}\Omega}$,
and three natural properties of ranks that $R^{(0)}$ satisfies:
\begin{enumerate}
\item $R^{(0)}(f\circ g;\Omega)\leq\min\left\{ R^{(0)}(f),R^{(0)}(g)\right\} $,
\item $R^{(0)}(f+g;\Omega)\leq R^{(0)}(f)+R^{(0)}(g)$,
\item $R^{(0)}(x\mapsto Ax;\Omega)=\mathrm{Rank}A$ for any full dimensional
and bounded $\Omega$.
\end{enumerate}
These results imply that for any function $f=\phi\circ A\circ\psi$
that is linear up to bijections $\phi,\psi$, the conjecture is true
$R^{(0)}(f;\Omega)=\mathrm{Rank}_{BN}(f;\Omega)=\mathrm{Rank}A$.

The proof of the aforementioned sandwich bound in \cite{jacot_2022_BN_rank}
actually prove an upper bound of the form $L\mathrm{Rank}_{BN}(f;\Omega)+O(1)$
thus proving that the $R^{(1)}$ term is upper bounded. The following
theorem proves a lower bound on $R^{(1)}$ as well as some of its
properties:
\begin{thm}
\label{thm:properties_first_correction}For all inputs $x$ where
$\mathrm{Rank}Jf(x)=R^{(0)}(x)$, $R^{(1)}(f)\geq2\log\left|Jf(x)\right|_{+}$,
furthermore:
\begin{enumerate}
\item If $R^{(0)}(f\circ g)=R^{(0)}(f)=R^{(0)}(g)$, then $R^{(1)}(f\circ g)\leq R^{(1)}(f)+R^{(1)}(g)$.
\item If $R^{(0)}(f+g)=R^{(0)}(f)+R^{(0)}(g)$, then $R^{(1)}(f+g)\leq R^{(1)}(f)+R^{(1)}(g)$.
\item If $P_{\mathrm{Im}A^{T}}\Omega$ and $A\Omega$ are $k=\mathrm{Rank}A$
dimensional and completely positive (i.e. they can be embedded isometrically
into $\mathbb{R}_{+}^{m}$ for some $m$), then $R^{(1)}(x\mapsto Ax;\Omega)=2\log\left|A\right|_{+}$
.
\end{enumerate}
\end{thm}

Notice how these properties clearly point to the first correction
$R^{(1)}(f)$ measuring a notion of regularity of $f$ instead of
a notion of rank. One can think of $L_{2}$-regularized deep nets
as learning functions $f$ that minimize
\[
\min_{f}C(f(X))+\lambda LR^{(0)}(f)+\lambda R^{(1)}(f).
\]
The depth determines the balance between the rank regularization and
regularity regularization. Without the $R^{(1)}$-term, the above
minimization would never be unique since there can be multiple functions
$f$ with the same training outputs $f(X)$ with the same rank $R^{(0)}(f)$.

Under the assumption that $R^{(0)}$ only takes integer value the
above optimization can be rewritten as 
\[
\min_{k=0,1,\dots}\lambda Lk+\min_{f:R^{(0)}(f)=k}C(f(X))+\lambda R^{(1)}(f).
\]
Every inner minimization for a rank $k$ that is attained inside the
set $\left\{ f:R^{(0)}(f)=k\right\} $ corresponds to a different
local minimum. Note how these inner minimization do not depend on
the depth, suggesting the existence of sequences of local minima for
different depths that all represent approximately the same function
$f$, as can be seen in Figure \ref{fig:norm_vs_depth}. We can classify
these minima according to whether they recover the true BN-rank $k^{*}$
of the task, underestimate or overestimate it. 

In linear networks \cite{wang_2023_bias_SGD_L2}, rank underestimating
minima cannot fit the data, but it is always possible to fit any data
with a BN-rank 1 function (by mapping injectively the datapoints to
a line and then mapping them nonlinearly to any other configuration).
We therefore need to also differentiate between rank-underestimating
minima that fit or do not fit the data. The non-fitting minima can
in theory be avoided by taking a small enough ridge (along the lines
of \cite{wang_2023_bias_SGD_L2}), but we do observe them empricially
for large depths in Figure \ref{fig:norm_vs_depth}. 

In contrast, we have never observed fitting rank-underestimating minima,
though their existence was proven for large enough depths in \cite{jacot_2022_BN_rank}.
A possible explanation for why GD avoids these minima is their $R^{(1)}$
value explodes with the number of datapoints $N$, since these network
needs to learn a space filling surface (a surface of dimension $k<k^{*}$
that visits random outputs $y_{i}$ that are sampled from a $k^{*}$-dimensional
distribution). More precisely Theorem 2 of \cite{jacot_2022_BN_rank}
implies that the $R^{(1)}$ value of fitting BN-rank 1 minima explodes
at a rate of $2(1-\frac{1}{k^{*}})\log N$ as the number of datapoints
$N$ grows, which could explain why we rarely observe such minima
in practice, but another explanation could be that these minima are
very narrow, as explained in Section \ref{subsec:Narrowness-of-rank-underestimating-minima}.

In our experiments we often encountered rank overestimating minima
and we are less sure about how to avoid them, though it seems that
increasing the depth helps (see Figure \ref{fig:norm_vs_depth}),
and that SGD might help too by analogy with the linear case \cite{wang_2023_bias_SGD_L2}.
Thankfully overestimating the rank is less problematic for generalization,
as supported by the fact that it is possible to approximate BN-rank
$k^{*}$ with a higher rank function with any accuracy, while doing
so with a low rank function requires a pathological function.

\begin{figure}
\subfloat[Parameter norm and depth]{\includegraphics[scale=0.45]{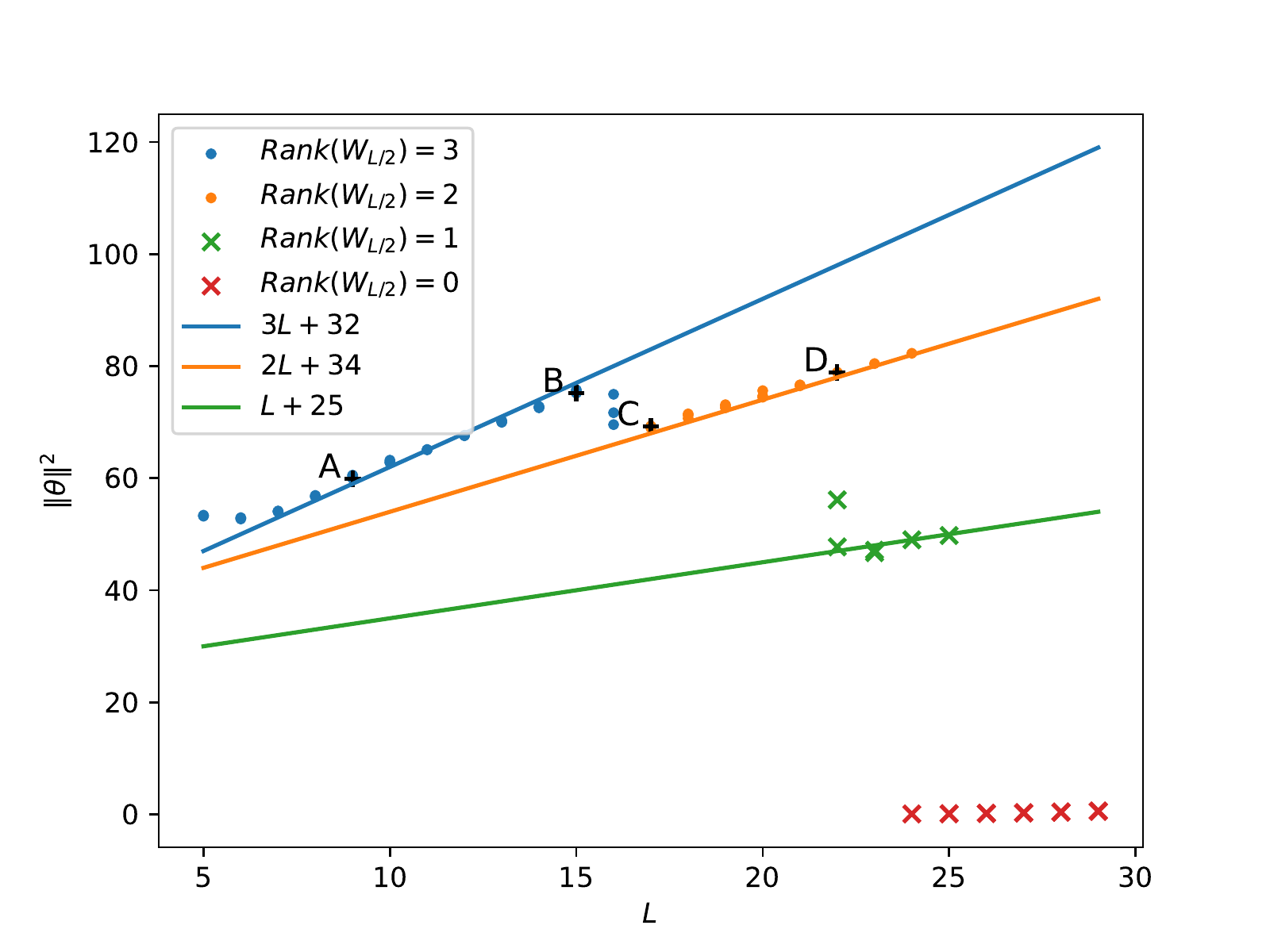}

}\hspace{-0.5cm}\subfloat[Bottleneck structure at different depths.]{\includegraphics[scale=0.45]{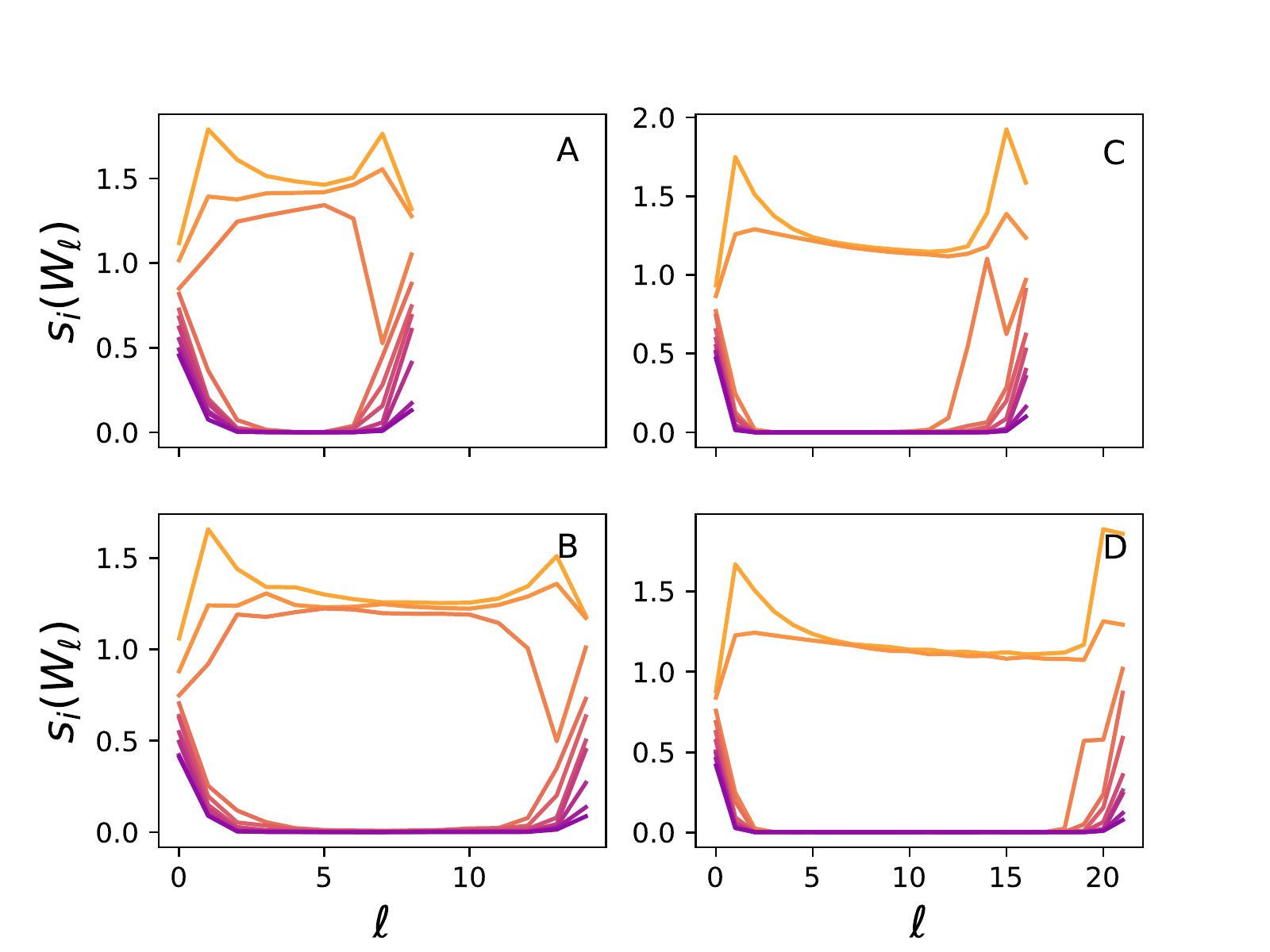}

}\hspace{-0.5cm}

\caption{\label{fig:norm_vs_depth}(a) Plot of the parameter norm at the end
of training ($\lambda=0.001$) over a range of depths, colored acoording
to the rank (\# of sing. vals above 0.1) of the weight matrices $W_{\nicefrac{L}{2}}$
in the middle of the network, and marked with a dot '.' or cross 'x'
depending on whether the final train cost is below or above 0.1. The
training data is synthetic and designed to have a optimal rank $k^{*}=2$.
We see different ranges of depth where the network converges to different
rank, with larger depths leading to smaller rank, until training fails
and recover the zero parameters for $L>25$. Within each range the
norm $\left\Vert \theta\right\Vert ^{2}$ is well approximated by
a affine function with slope equal to the rank. (b) Plot of the singular
values of $W_{\ell}$ throughout the networks for 4 trials, we see
that the bottleneck structure remains essentially the same throughout
each range of depth, with only the middle low-rank part growing with
the depth.}
\end{figure}

\subsection{Second Correction}

We now identify a few properties of the second correction $R^{(2)}$:
\begin{prop}
\label{prop:second-correction-properties}If there is a limiting representation
as $L\to0$ in the optimal representation of $f$, then $R^{(2)}(f)\geq0$.
Furthermore:
\begin{enumerate}
\item If $R^{(0)}(f\circ g)=R^{(0)}(f)=R^{(0)}(g)$ and $R^{(1)}(f\circ g)=R^{(1)}(f)+R^{(1)}(g)$,
then $\sqrt{R^{(2)}(f\circ g)}\leq\sqrt{R^{(2)}(f)}+\sqrt{R^{(2)}(g)}$.
\item If $R^{(0)}(f+g)=R^{(0)}(f)+R^{(0)}(g)$ and $R^{(1)}(f+g)=R^{(1)}(f)+R^{(1)}(g)$,
then $R^{(2)}(f+g)\leq R^{(2)}(f)+R^{(2)}(g)$.
\item If $A^{p}\Omega$ is $k=\mathrm{Rank}A$-dimensional and completely
positive for all $p\in[0,1]$, where $A^{p}$ has its non-zero singular
taken to the $p$-th power, then $R^{(2)}(x\mapsto Ax;\Omega)=\frac{1}{2}\left\Vert \log_{+}A^{T}A\right\Vert ^{2}$.
\end{enumerate}
\end{prop}

While the properties are very similar to those of $R^{(1)}$, the
necessary conditions necessary to apply them are more restrictive.
There might be case where the first two terms $R^{(0)}$ and $R^{(1)}$
do not uniquely determine a minimum, in which case the second correction
$R^{(2)}$ needs to be considered.

In linear networks the second correction $R^{(2)}(A)=\frac{1}{2}\left\Vert \log_{+}A^{T}A\right\Vert ^{2}$
plays an important role, as it bounds the operator norm of $A$ (which
is not bounded by $R^{(0)}(A)=\mathrm{Rank}A$ nor $R^{(1)}(A)=2\log\left|A\right|_{+}$),
thus guaranteeing the convergence of the hidden representations in
the middle of network. We hoped at first that $R^{(2)}$ would have
similar properties in the nonlinear case, but we were not able to
prove anything of the sort. Actually in contrast to the linear setting,
the representations of the network can diverge as $L\to\infty$, which
explains why the $R^{(2)}$ does not give any similar control, which
would guarantee convergence.

\subsection{Representation geodesics}

One can think of the sequence of hidden representations $\tilde{\alpha}_{1},\dots,\tilde{\alpha}_{L}$
as a path from the input representation to the output representation
that minimizes the weight norm $\left\Vert W_{\ell}\right\Vert ^{2}$
required to map from one representation to the next. As the depth
$L$ grows, we expect this sequence to converge to a form of geodesic
in representation space. Such an analysis has been done in \cite{owhadi2020_L2_neuralODE}
for ResNet where these limiting geodesics are continuous. 

Two issues appear in the fully-connected case. First a representation
$\tilde{\alpha}_{\ell}$ remains optimal after any swapping of its
neurons or other symmetries, but this can easily be solved by considering
representations $\tilde{\alpha}_{\ell}$ up to orthogonal transformation,
i.e. to focus on the kernels $K_{\ell}(x,y)=\tilde{\alpha}_{\ell}(x)^{T}\tilde{\alpha}_{\ell}(y)$.
Second the limiting geodesics of fully-connected networks are not
continuous, and as such they cannot be described by a local metric.

We therefore turn to the representation cost of DNNs to describe the
hidden representations of the network, since the $\ell$-th pre-activation
function $\tilde{\alpha}^{(\ell)}:\Omega\to\mathbb{R}^{n_{\ell}}$
in a network which minimizes the parameter norm must satisfy
\[
R(f;\Omega,L)=R(\tilde{\alpha}_{\ell};\Omega,\ell)+R(\sigma(\tilde{\alpha}_{\ell})\to f;\Omega,L-\ell).
\]
Thus the limiting representations $\tilde{\alpha}_{p}=\lim_{L\to\infty}\tilde{\alpha}_{\ell_{L}}$
(for a sequence of layers $\ell_{L}$ such that $\lim_{L\to\infty}\nicefrac{\ell_{L}}{L}=p\in(0,1)$)
must satisfy 
\begin{align*}
R^{(0)}(f;\Omega) & =R^{(0)}(\tilde{\alpha}_{p};\Omega)=R^{(0)}(\sigma(\tilde{\alpha}_{p})\to f;\Omega)\\
R^{(1)}(f;\Omega) & =R^{(1)}(\tilde{\alpha}_{p};\Omega)+R^{(1)}(\sigma(\tilde{\alpha}_{p})\to f;\Omega)
\end{align*}

Let us now assume that the limiting geodesic is continuous at $p$
(up to orthogonal transformation, which do not affect the representation
cost), meaning that any other sequence of layers $\ell'_{L}$ converging
to the same ratio $p\in(0,1)$ would converge to the same representation.
The taking the limits with two sequences $\lim\frac{\ell_{L}}{L}=p=\lim\frac{\ell'_{L}}{L}$
such that $\lim\ell'_{L}-\ell_{L}=+\infty$ and and taking the limit
of the equality 
\[
R(f;\Omega,L)=R(\tilde{\alpha}_{\ell_{L}};\Omega,\ell_{L})+R(\sigma(\tilde{\alpha}_{\ell_{L}})\to\tilde{\alpha}_{\ell'_{L}};\Omega,\ell_{L}'-\ell_{L})+R(\sigma(\tilde{\alpha}_{\ell'_{L}})\to f;\Omega,L-\ell_{L}'),
\]
we obtain that $R^{(0)}(\sigma\left(\tilde{\alpha}_{p}\right)\to\tilde{\alpha}_{p};\Omega)=R^{(0)}(f)$
and $R^{(1)}(\sigma\left(\tilde{\alpha}_{p}\right)\to\tilde{\alpha}_{p};\Omega)=0$.
This implies that $\sigma(\tilde{\alpha}(x))=\tilde{\alpha}(x)$ at
any point $x$ where $\mathrm{Rank}Jf(x)=R^{(0)}(f;\Omega)$, thus
$R^{(0)}(id;\tilde{\alpha}_{p}(\Omega))=R^{(0)}(f;\Omega)$ and $R^{(1)}(id;\tilde{\alpha}_{p}(\Omega))=0$
if $\mathrm{Rank}Jf(x)=R^{(0)}(f;\Omega)$ for all $x\in\Omega$.

\subsubsection{Identity}

When evaluated on the identity, the first two terms $R^{(0)}(id;\Omega)$
and $R^{(1)}(id;\Omega)$ describe properties of the domain $\Omega$.

For any notion of rank, $\mathrm{Rank}(id;\Omega)$ defines a notion
of dimensionality of $\Omega$. The Jacobian rank $\mathrm{Rank}_{J}(id;\Omega)=\max_{x\in\Omega}\dim T_{x}\Omega$
is the maximum tangent space dimension, while the Bottleneck rank
$\mathrm{Rank}_{BN}(id;\Omega)$ is the smallest dimension $\Omega$
can be embedded into. For example, the circle $\Omega=\mathbb{S}^{2-1}$
has $\mathrm{Rank}_{J}(id;\Omega)=1$ and $\mathrm{Rank}_{BN}(id;\Omega)=2$.

On a domain $\Omega$ where the two notions of dimensionality match
$\mathrm{Rank}_{J}(id;\Omega)=\mathrm{Rank}_{BN}(id;\Omega)=k$, the
first correction $R^{(1)}(id;\Omega)$ is non-negative since for any
$x$ with $\dim T_{x}\Omega=k$, we have $R^{(1)}(id;\Omega)\geq\log\left|P_{T_{x}}\right|_{+}=0.$
The $R^{(1)}(id;\Omega)$ value measures how non-planar the domain
$\Omega$ is, being $0$ only if $\Omega$ is $k$-planar, i.e. its
linear span is $k$-dimensional:
\begin{prop}
\label{prop:R1_identity}For a domain with $\mathrm{Rank}_{J}(id;\Omega)=\mathrm{Rank}_{BN}(id;\Omega)=k$,
then $R^{(1)}(id;\Omega)=0$ if and only if $\Omega$ is $k$-planar
and completely positive.
\end{prop}

This proposition shows that the $R^{(1)}$ term does not only bound
the Jacobian of $f$ as shown in Theorem \ref{thm:properties_first_correction},
but also captures properties of the curvature of the domain/function.

Thus at ratios $p$ where the representation geodesics converge continuously,
the representations $\tilde{\alpha}_{p}(\Omega)$ are $k=R^{(0)}(f;\Omega)$-planar,
proving the Bottleneck structure that was only observed empirically
in \cite{jacot_2022_BN_rank}. But the assumption of convergence over
which we build this argument does not hold in general, actually we
give in the appendix an example of a simple function $f$ whose optimal
representations diverges in the infinite depth limit. This is in stark
contrast to the linear case, where the second correction $R^{(2)}(A)=\frac{1}{2}\left\Vert \log_{+}A^{T}A\right\Vert ^{2}$
guarantees convergence, since it bounds the operator norm of $A$.
To prove and describe the bottleneck structure in nonlinear DNNs,
we therefore need to turn to another strategy.

\section{Bottleneck Structure in Large Depth Networks\label{sec:Bottleneck-Structure}}

Up to now we have focused on one aspect of the Bottleneck structure
observed in \cite{jacot_2022_BN_rank}: that the representations $\alpha_{\ell}(X)$
inside the Bottleneck are approximately $k$-planar. But another characteristic
of this phenomenon is that the weight matrices $W_{\ell}$ in the
bottleneck have $k$ dominating singular values, all close to $1$.
This property does not require the convergence of the geodesics and
can be proven with finite depth rates:
\begin{thm}
\label{thm:BN-structure-weights}Given parameters $\theta$ of a depth
$L$ network, with $\left\Vert \theta\right\Vert ^{2}\leq kL+c_{1}$
and a point $x$ such that $\mathrm{Rank}Jf_{\theta}(x)=k$, then
there are $w_{\ell}\times k$ (semi-)orthonormal $U_{\ell},V_{\ell}$
such that $\sum_{\ell=1}^{L}\left\Vert W_{\ell}-U_{\ell}V_{\ell-1}^{T}\right\Vert _{F}^{2}\leq c_{1}-2\log\left|Jf_{\theta}(x)\right|_{+}$
thus for any $p\in(0,1)$ there are at least $(1-p)L$ layers $\ell$
with
\[
\left\Vert W_{\ell}-U_{\ell}V_{\ell-1}^{T}\right\Vert _{F}^{2}\leq\frac{c_{1}-2\log\left|Jf_{\theta}(x)\right|_{+}}{pL}.
\]
\end{thm}

Note how we not only obtain finite depth rates, but our result has
the advantage of being applicable to any parameters with a sufficiently
small parameter norm (close to the minimal norm solution). The bound
is tighter at optimal parameters in which case $c_{1}=R^{(1)}(f_{\theta})$,
but the theorem shows that the Bottleneck structure generalizes to
points that are only almost optimal.

To prove that the pre-activations $\tilde{\alpha}_{\ell}(X)$ are
approximately $k$-dimensional for some dataset $X$ (that may or
may not be the training set) we simply need to show that the activations
$\alpha_{\ell-1}(X)$ do not diverge, since $\tilde{\alpha}_{\ell}(X)=W_{\ell}\alpha_{\ell-1}(X)+b_{\ell}$
(and one can show that the bias will be small at almost every layer
too). By our counterexample we know that we cannot rule out such explosion
in general, however if we assume that the NTK \cite{jacot2018neural}
$\Theta^{(L)}(x,x)$ is of order $O(L)$, then we can guarantee to
convergence of the activations $\alpha_{\ell-1}(X)$ at almost every
layer:
\begin{thm}
\label{thm:NTK-implies-bounded-representations}Given balanced parameters
$\theta$ of a depth $L$ network, with $\left\Vert \theta\right\Vert ^{2}\leq kL+c_{1}$
and a point $x$ such that $\mathrm{Rank}Jf_{\theta}(x)=k$ then if
$\mathrm{Tr}\left[\Theta^{(L)}(x,x)\right]\leq cL$, then $\sum_{\ell=1}^{L}\left\Vert \alpha_{\ell-1}(x)\right\Vert _{2}^{2}\leq\frac{c\max\{1,e^{\frac{c_{1}}{k}}\}}{k\left|Jf_{\theta}(x)\right|_{+}^{\nicefrac{2}{k}}}L$
and thus for all $p\in(0,1)$ there are at least $(1-p)L$ layers
such that 
\[
\left\Vert \alpha_{\ell-1}(x)\right\Vert _{2}^{2}\leq\frac{1}{p}\frac{c\max\{1,e^{\frac{c_{1}}{k}}\}}{k\left|Jf_{\theta}(x)\right|_{+}^{\nicefrac{2}{k}}}.
\]
\end{thm}

Note that the balancedness assumption is not strictly necessary and
could easily be weakened to some form of approximate balancedness,
since we only require the fact that the parameter norm $\left\Vert W_{\ell}\right\Vert _{F}^{2}$
is well spread out throughout the layers, which follows from balancedness.

The NTK describes the narrowness of the minima \cite{jacot2019asymptotic},
and the assumption of bounded NTK is thus related to stability under
large learning rates. There are multiple notions of narrowness that
have been considered: 
\begin{itemize}
\item The operator norm of the Hessian $H$ (which is closely related to
the top eigenvalue of the NTK Gram matrix $\Theta^{(L)}(X,X)$ especially
in the MSE case where at any interpolating function $\left\Vert H\right\Vert _{op}=\frac{1}{N}\left\Vert \Theta^{(L)}(X,X)\right\Vert _{op}$)
which needs to be bounded by $\nicefrac{2}{\eta}$ to have convergence
when training with gradient descent with learning rate $\eta$.
\item The trace of the Hessian (in the MSE case $\mathrm{Tr}H=\frac{1}{N}\mathrm{Tr}\Theta^{(L)}(X,X)$)
which has been shown to describe the bias of stochastic gradient descent
or approximation thereof \cite{damian2021_label_noise_trace_Hessian,li2021_SGD_zero_loss}.
\end{itemize}
Thus boundedness of almost all activations as $L\to\infty$ can be
guaranteed by assuming either $\frac{1}{N}\left\Vert \Theta^{(L)}(X,X)\right\Vert _{op}\leq cL$
(which implies $d_{out}\mathrm{Tr}\Theta^{(L)}(X,X)\leq cL$) or $\frac{1}{N}\mathrm{Tr}\Theta^{(L)}(X,X)\leq cL$
directly, corresponding to either gradient descent with $\eta=\nicefrac{2}{cL}$
or stochastic gradient descent with a similar scaling of $\eta$).

Note that one can find parameters that learn a function with a NTK
that scales linearly in depth, but it is not possible to represent
non-trivial functions with a smaller NTK $\Theta^{(L)}\ll L$. This
is why we consider a linear scaling in depth to be the `normal' size
of the NTK, and anything larger to be narrow.

Putting the two Theorems together, we can prove the Bottleneck structure
for almost all representations $\tilde{\alpha}_{\ell}(X)$:
\begin{cor}
\label{cor:BN-structure-alpha}Given balanced parameters $\theta$
of a depth $L$ network with $\left\Vert \theta\right\Vert ^{2}\leq kL+c_{1}$
and a set of points $x_{1},\dots,x_{N}$ such that $\mathrm{Rank}Jf_{\theta}(x_{i})=k$
and $\frac{1}{N}\mathrm{Tr}\left[\Theta^{(L)}(X,X)\right]\leq cL$,
then for all $p\in(0,1)$ there are at least $(1-p)L$ layers such
that 
\begin{align*}
s_{k+1}\left(\frac{1}{\sqrt{N}}\tilde{\alpha}_{\ell}(X)\right) & \leq\sqrt{c_{1}-2\log\left|Jf_{\theta}(x)\right|_{+}}\left(\sqrt{\frac{1}{N}\sum_{i=1}^{N}\frac{c\max\{1,e^{\frac{c_{1}}{k}}\}}{k\left|Jf_{\theta}(x_{i})\right|_{+}^{\nicefrac{2}{k}}}}+\sqrt{p}\right)\frac{1}{p\sqrt{L}}.
\end{align*}
\end{cor}

\subsection{Narrowness of rank-underestimating minima\label{subsec:Narrowness-of-rank-underestimating-minima}}

We know that the large $R^{(1)}$ value of BN-rank 1 fitting functions
is related to the explosion of its derivative, but a large Jacobian
also leads to a blow up of the NTK:
\begin{prop}
\label{prop:blowup-NTK-Jacobian}For any point $x$, we have
\[
\left\Vert \partial_{xy}^{2}\Theta(x,x)\right\Vert _{op}\geq2L\left\Vert Jf_{\theta}(x)\right\Vert _{op}^{2-\nicefrac{2}{L}}
\]
where $\partial_{xy}^{2}\Theta(x,x)$ is understood as a $d_{in}d_{out}\times d_{in}d_{out}$
matrix.

Furthermore, for any two points $x,y$ such that the pre-activations
of all neurons of the network remain constant on the segment $[x,y]$,
then either $\left\Vert \Theta(x,x)\right\Vert _{op}$ or $\left\Vert \Theta(y,y)\right\Vert _{op}$
is lower bounded by $\frac{L}{4}\left\Vert x-y\right\Vert ^{2}\left\Vert Jf_{\theta}(x)\frac{y-x}{\left\Vert x-y\right\Vert }\right\Vert _{2}^{2-\nicefrac{2}{L}}.$
\end{prop}

With some additional work, we can show the the NTK of such rank-underestimating
functions will blow up, suggesting a narrow minimum:
\begin{thm}
\label{thm:BN-rank-1-fitting-NTK}Let $f^{*}:\Omega\to\mathbb{R}^{d_{out}}$
be a function with Jacobian rank $k^{*}>1$ (i.e. there is a $x\in\Omega$
with $\mathrm{Rank}Jf^{*}(x)=k^{*}$), then with high probability
over the sampling of a training set $x_{1},\dots,x_{N}$ (sampled
from a distribution with support $\Omega$), we have that for any
parameters $\theta$ of a deep enough network that represent a BN-rank
1 function $f_{\theta}$ that fits the training set $f_{\theta}(x_{i})=f^{*}(x_{i})$
with norm $\left\Vert \theta\right\Vert ^{2}=L+c_{1}$ then there
is a point $x\in\Omega$ where the NTK satisfies
\[
\mathrm{Tr}\left[\Theta^{(L)}(x,x)\right]\geq c''Le^{-c_{1}}N^{4-\frac{4}{k^{*}}}.
\]
\end{thm}

One the one hand, we know that $c_{1}$ must satisfy $c_{1}\geq R^{(1)}(f_{\theta})\geq2(1-\frac{1}{k^{*}})\log N$
but if $c_{1}$ is within a factor of 2 of this lower bound $c_{1}<4(1-\frac{1}{k^{*}})\log N$,
then the above shows that the NTK will blow up a rate $N^{\alpha}L$
for a positive $\alpha$.

The previous explanation for why GD avoids BN-rank 1 fitting functions
was that when $N$ is much larger than the depth $L$ (exponentially
larger), there is a rank-recovering function with a lower parameter
norm than any rank-underestimating functions. But this relies on the
assumption that GD converges to the lower norm minima, and it is only
true for sufficiently small depths. In contrast the narrowness argument
applies for any large enough depth and does not assume global convergence.

Of course the complete explanation might be a mix of these two reasons
and possbily some other phenomenon too. Proving why GD avoids minima
that underestimate the rank with a rank $1<k<k^{*}$ also remains
an open question.

\begin{figure}
\hspace{-0.6cm}\subfloat[Spectra of the $W_{\ell}$s]{\includegraphics[scale=0.35]{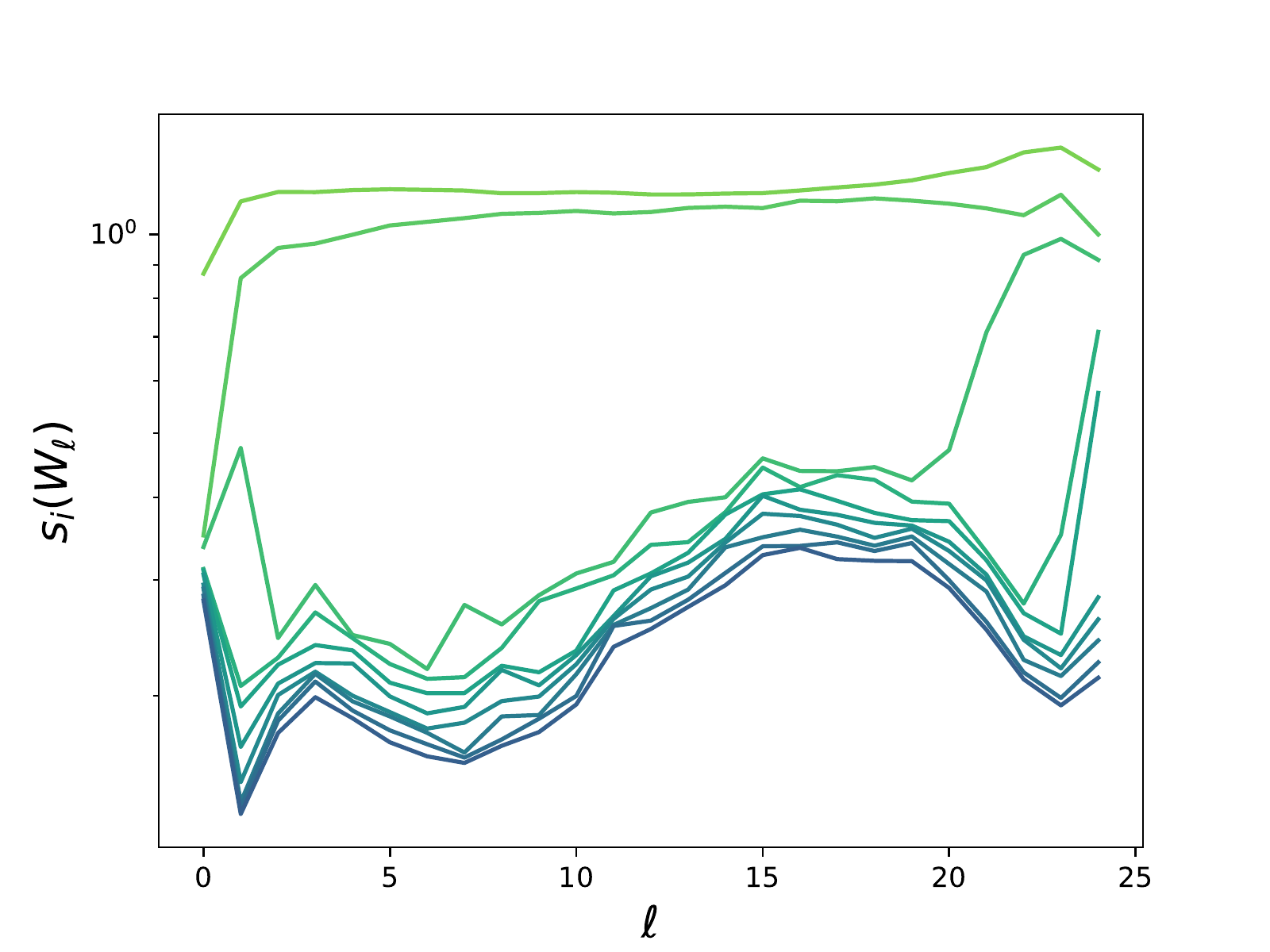}

}\hspace{-0.5cm}\subfloat[PCA of $\alpha_{6}(X)$.]{\includegraphics[scale=0.35]{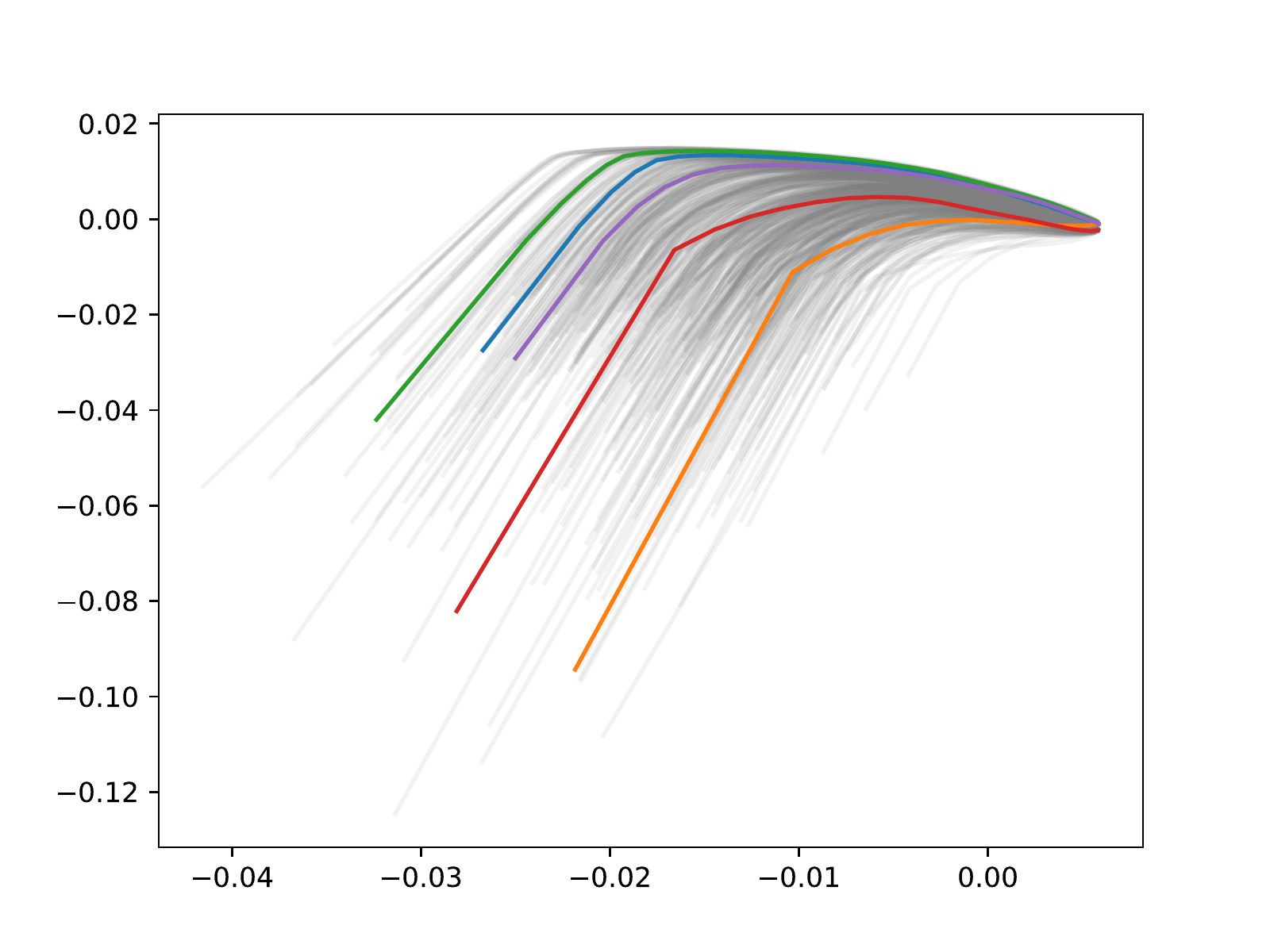}

}\hspace{-0.5cm}\subfloat[Dist. $\left\Vert \alpha_{2}(x)-\alpha_{2}(x_{0})\right\Vert $.]{\includegraphics[scale=0.32]{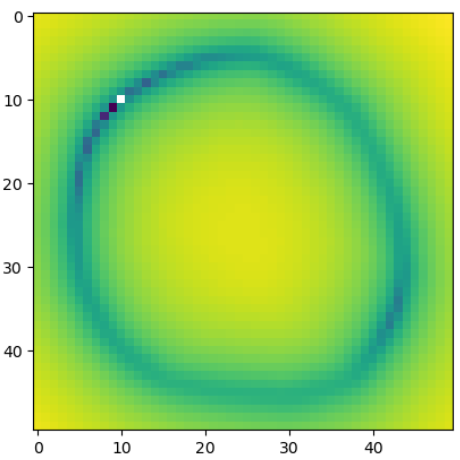}

}\hspace{-0.5cm}

\caption{\label{fig:symmetry_learning}A depth $L=25$ network with a width
of $200$ trained on the task described in Section \ref{subsec:Symmery-Learning}
with a ridge $\lambda=0.0002$. (a) Singular values of the weight
matrices of the network, showing two outliers in the bottleneck, which
implies that the network has recovered the true rank of 2. (b) Hidden
representation of the $6$-th layer projected to the first two dimensions,
we see how images of GD paths do not cross in this space, showing
that the dynamics on these two dimensions are self-consistent. (c)
The distance $\left\Vert \alpha_{2}(x_{0})-\alpha_{2}(x)\right\Vert $
in the second hidden layer between the representations at a fixed
point $x_{0}$ (at the white pixel) and another point $x$ on a plane
orthogonal to the axis $w$ of rotation, we see that all points on
the same symmetry orbit are collapsed together, proving that the network
has learned the rotation symmetry.}
\end{figure}

\section{Numerical Experiment: Symmetry Learning\label{subsec:Symmery-Learning}}

In general, functions with a lot of symmetries have low BN-rank since
a function $f$ with symmetry group $G$ can be decomposed as mapping
the inputs $\Omega$ to the inputs module symmetries $\nicefrac{\Omega}{G}$
and then mapping it to the outputs, thus $\mathrm{Rank}_{BN}(f;\Omega)\leq\dim\nicefrac{\Omega}{G}$
where $\dim\nicefrac{\Omega}{G}$ is the smallest dimension $\nicefrac{\Omega}{G}$
can be embedded into. Thus the bias of DNNs to learn function with
a low BN-rank can be interpreted as the network learning symmetries
of the task. With this interpretation, overestimating the rank corresponds
to failing to learn all symmetries of the task, while underestimating
the rank can be interpreted as the network learning spurious symmetries
that are not actual symmetries of the task.

To test this idea, we train a network to predict high dimensional
dynamics with high dimensional symmetries. Consider the loss $C(v)=\left\Vert vv^{T}-\left(ww^{T}+E\right)\right\Vert _{F}^{2}$
where $w\in\mathbb{R}^{d}$ is a fixed unit vector and $E$ is a small
noise $d\times d$ matrix. We optimize $v$ with gradient descent
to try and fit the true vector $w$ (up to a sign). One can think
of these dynamics as learning a shallow linear network $vv^{T}$ with
a single hidden neuron. We will train a network to predict the evolution
of the cost in time $C(v(t))$.

For small noise matrix $E$, the GD dynamics of $v(t)$ are invariant
under rotation around the vector $w$. As a result, the high-dimensional
dynamics of $v(t)$ can captured by only two \emph{summary statistics
}$u(v)=((w^{T}v)^{2},\left\Vert (I-ww^{T})v\right\Vert ^{2})$: the
first measures the position along the axis formed by $w$ and the
second the distance to this axis \cite{arous2022summary_stats}. The
evolution of the summary statistics is (approximately) self-consistent
(using the fact that $\left\Vert v\right\Vert ^{2}=(w^{T}v)^{2}+\left\Vert (I-ww^{T})v\right\Vert ^{2}$):
\begin{align*}
\partial_{t}(w^{T}v)^{2} & =-8(\left\Vert v\right\Vert ^{2}-1)(w^{T}v)^{2}+O(\left\Vert E\right\Vert )\\
\partial_{t}\left\Vert (I-ww^{T})v\right\Vert ^{2} & =-8\left\Vert v\right\Vert ^{2}\left\Vert (I-ww^{T})v\right\Vert ^{2}+O(\left\Vert E\right\Vert ).
\end{align*}

Our goal now is to see whether a DNN can learn these summary statistics,
or equivalently learn the underlying rotation symmetry. To test this,
we train a network on the following supervised learning problem: given
the vector $v(0)$ at initialization, predict the loss $\left(C(v(1)),\dots,C(v(T))\right)$
over the next $T$ GD steps. For $E=0$ , the function $f^{*}:\mathbb{R}^{d}\to\mathbb{R}^{T}$
that is to be learned has BN-rank 2, since one can first map $v(0)$
to the corresponding summary statistics $u(v(0))\in\mathbb{R}^{2}$,
and then solve the differential equation on the summary statistics
$(u(1),\dots,u(T))$ over the next $T$ steps, and compute the cost
$C(v)=\left\Vert v\right\Vert ^{4}-2(w^{T}v)^{2}+1+O(\left\Vert E\right\Vert )$
from $u$. 

We observe in Figure \ref{fig:symmetry_learning} that a large depth
$L_{2}$-regularized network trained on this task learns the rotation
symmetry of the task and learns two dimensional hidden representations
that are summary statistics (summary statistics are only defined up
to bijections, so the learned representation match $u(v)$ only up
to bijection but they can be recognized from the fact that the GF
paths do not cross on the 2D representation).

\section{Conclusion}

We have computed corrections to the infinite depth description of
the representation cost of DNNs given in \cite{jacot_2022_BN_rank},
revealing two regularity $R^{(1)},R^{(2)}$ measures that balance
against the dominating low rank/dimension bias $R^{(0)}$. We have
also partially described another regularity inducing bias that results
from large learning rates. We argued that these regularity bias play
a role in stopping the network from underestimating the `true' BN-rank
of the task (or equivalently overfitting symmetries).

We have also proven the existence of a bottleneck structure in the
weight matrices and under the condition of a bounded NTK of the learned
representations, where most hidden representations are approximately
$k=R^{(0)}(f_{\theta})$-dimensional, with only a few high-dimensional
representations.

\subsection*{Acknowledgement}

Thank you to Peter Sukenik who identified an error in the proof of
Theorem \ref{thm:BN-structure-weights} and proposed a simple fix.

\bibliographystyle{alpha}
\bibliography{./../../../projects/main}

\newpage{}

\appendix

\section{Properties of the Corrections}

\subsection{First Correction}
\begin{thm}[Theorem \ref{thm:properties_first_correction} from the main]
\label{thm:properties_first_correction-1}For all inputs $x$ where
$\mathrm{Rank}Jf(x)=R^{(0)}(f;\Omega)$, we have $R^{(1)}(f)\geq2\log\left|Jf(x)\right|_{+}$,
furthermore:
\begin{enumerate}
\item If $R^{(0)}(f\circ g)=R^{(0)}(f)=R^{(0)}(g)$, then $R^{(1)}(f\circ g)\leq R^{(1)}(f)+R^{(1)}(g)$.
\item If $R^{(0)}(f+g)=R^{(0)}(f)+R^{(0)}(g)$, then $R^{(1)}(f+g)\leq R^{(1)}(f)+R^{(1)}(g)$.
\item If $P_{\mathrm{Im}A^{T}}\Omega$ and $A\Omega$ are $k=\mathrm{Rank}A$
dimensional and completely positive (i.e. they can be embedded with
an isometric linear map into $\mathbb{R}_{+}^{m}$ for some $m$),
then $R^{(1)}(x\mapsto Ax;\Omega)=2\log\left|A\right|_{+}$ .
\end{enumerate}
\end{thm}

\begin{proof}
For the first bound, we remember that $R(f;\Omega,L)\geq L\left\Vert Jf\right\Vert _{\nicefrac{2}{L}}^{\nicefrac{2}{L}}$,
therefore
\begin{align*}
R^{(1)}(f;\Omega) & =\lim_{L\to\infty}R(f;\Omega,L)-LR^{(0)}(f;\Omega)\\
 & \geq\lim_{L\to\infty}L\sum_{i=1}^{\mathrm{Rank}Jf(x)}s_{i}(Jf(x))^{\frac{2}{L}}-1\\
 & \geq\sum_{i=1}^{\mathrm{Rank}Jf(x)}2\log s_{i}(Jf(x))
\end{align*}
where we used $s^{\frac{2}{L}}-1=e^{\frac{2}{L}\log s}-1\geq\frac{2}{L}\log s$.

(1) Since $R(f\circ g;\Omega,L_{1}+L_{2})\leq R(f;L_{1})+R(g;L_{2})$,
we have
\begin{align*}
R^{(1)}(f\circ g;\Omega) & =\lim_{L_{1}+L_{2}\to\infty}R(f\circ g;\Omega,L_{1}+L_{2})-(L_{1}+L_{2})R^{(0)}(f\circ g;\Omega)\\
 & \leq\lim_{L_{1}\to\infty}R(f;\Omega,L_{1})-L_{1}R^{(0)}(f;\Omega)+\lim_{L_{2}\to\infty}R(g;\Omega,L_{2})-L_{2}R^{(0)}(f;\Omega)\\
 & =R^{(1)}(f;\Omega)+R^{(1)}(g;\Omega).
\end{align*}

(2) Since $R(f+g;\Omega,L)\leq R(f;\Omega,L)+R(g;\Omega,L)$, we have
\begin{align*}
R^{(1)}(f+g;\Omega) & =\lim_{L\to\infty}R(f+g;\Omega,L)-LR^{(0)}(f+g;\Omega)\\
 & \leq\lim_{L\to\infty}R(f;\Omega,L)-LR^{(0)}(f;\Omega)+\lim_{L\to\infty}R(g;\Omega,L)-LR^{(0)}(g;\Omega)\\
 & =R^{(1)}(f;\Omega)+R^{(1)}(g;\Omega).
\end{align*}

(3) By the first bound, we know that $R^{(1)}(x\mapsto Ax;\Omega)\geq2\log\left|A\right|_{+}$,
we now need to show $R^{(1)}(x\mapsto Ax;\Omega)\leq2\log\left|A\right|_{+}$.
Let us define the set of completely positive representations as the
set of bilinear kernels $K(x,y)=x^{T}B^{T}By$ such that $Bx$ has
non-negative entries for all $x\in\Omega$ (we say that a kernel $K$
is completely positive over $\Omega$ if it can be represented in
this way for some choice of $B$). The set of completely positive
representations is convex, since for $K(x,y)=x^{T}B^{T}By$ and $\tilde{K}(x,y)=x^{T}\tilde{B}^{T}\tilde{B}y$,
we have 
\[
\frac{K(x,y)+\tilde{K}(x,y)}{2}=x^{T}\left(\begin{array}{c}
\frac{1}{\sqrt{2}}B\\
\frac{1}{\sqrt{2}}\tilde{B}
\end{array}\right)^{T}\left(\begin{array}{c}
\frac{1}{\sqrt{2}}B\\
\frac{1}{\sqrt{2}}\tilde{B}
\end{array}\right)y.
\]
The conditions that there are $O_{in}$ and $O_{out}$ with $O_{in}^{T}O_{in}=P_{\mathrm{Im}A^{T}}$
and $O_{out}^{T}O_{out}=P_{\mathrm{Im}A}$ such that $O_{in}\Omega\in\mathbb{R}_{+}^{k_{1}}$
and $O_{out}A\Omega\in\mathbb{R}_{+}^{k_{2}}$ is equivalent to saying
that the kernels $K_{in}(x,y)=x^{T}P_{\mathrm{Im}A^{T}}y$ and $K_{out}(x,y)=x^{T}A^{T}Ax$
are completely positive over $\Omega$.

By the convexity of completely positive representations, the interpolation
$K_{p}=pK_{in}+(1-p)K_{out}$ is completely positive for all $p\in[0,1]$.
Now choose for all depths $L$ and all layers $\ell=1,\dots,L-1$
a matrix $B_{L,\ell}$ such that $K_{p=\frac{\ell}{L}}(x,y)=x^{T}B_{L,\ell}^{T}B_{L,\ell}y$
and then choose the weights $W_{\ell}$ of the depth $L$ network
as 
\[
W_{\ell}=B_{L,\ell}B_{L,\ell-1}^{+},
\]
using the convention $B_{L,0}=I_{d_{in}}$ and $B_{L,L}=I_{out}$.
By induction, we show that for any input $x\in\Omega$ the activation
of the $\ell$-th hidden layer is $B_{L,\ell}x$. This is true for
$\ell=1$, since $W_{1}=B_{L,1}$ and therefore $p^{(1)}(x)=B_{L,1}x$
which has positive entries so that $q^{(1)}(x)=\sigma\left(p^{(1)}(x)\right)=B_{L,1}x$.
Then by induction 
\[
p^{(\ell)}(x)=W_{\ell}q^{(\ell-1)}(x)=B_{L,\ell}B_{L,\ell-1}^{+}B_{L,\ell-1}x=B_{L,\ell}x,
\]
which has positive entries, so that again $q^{(\ell)}(x)=\sigma\left(p^{(\ell)}(x)\right)=B_{L,\ell}x$.
In the end, we get $p^{(L)}(x)=Ax$ as needed.

Let us now compute the Frobenius norms of the weight matrices $\left\Vert W_{\ell}\right\Vert _{F}^{2}=\mathrm{Tr}\left[B_{L,\ell}^{T}B_{L,\ell}\left(B_{L,\ell-1}^{T}B_{L,\ell-1}\right)^{+}\right]$
as $L\to\infty$, remember that $B_{L,\ell}^{T}B_{L,\ell}=\frac{\ell}{L}P_{\mathrm{Im}A^{T}}+(1-\frac{\ell}{L})A^{T}A$,
therefore the matrices $B_{L,\ell}^{T}B_{L,\ell}$ and $B_{L,\ell-1}^{T}B_{L,\ell-1}$
converge to each other, so that at first order $B_{L,\ell}^{T}B_{L,\ell}\left(B_{L,\ell-1}^{T}B_{L,\ell-1}\right)^{+}$
converges to $P_{\mathrm{Im}A^{T}}$, so that $\left\Vert W_{\ell}\right\Vert _{F}^{2}\to\mathrm{Rank}A$,
so that $\sum_{\ell=1}^{L}\left\Vert W_{\ell}\right\Vert _{F}^{2}-L\mathrm{Rank}A$
converges to a finite value as $L\to\infty$. To obtain this finite
limit, we need to study approximate the next order 
\begin{align*}
\left\Vert W_{\ell}\right\Vert _{F}^{2}-\mathrm{Rank}A & =\sum_{i=1}^{\mathrm{Rank}A}2\log s_{i}(W_{i})+O(L^{-2})\\
 & =\log\left|B_{L,\ell}^{T}B_{L,\ell}\left(B_{L,\ell-1}^{T}B_{L,\ell-1}\right)^{+}\right|_{+}+O(L^{-2})\\
 & =\log\left|B_{L,\ell}^{T}B_{L,\ell}\right|_{+}-\log\left|B_{L,\ell-1}^{T}B_{L,\ell-1}\right|_{+}+O(L^{-2}).
\end{align*}
But as we sum all these second order terms, they cancel out, and we
are left with
\[
\sum_{\ell=1}^{L}\left\Vert W_{\ell}\right\Vert _{F}^{2}-L\mathrm{Rank}A=2\log\left|A\right|_{+}-2\log\left|I_{\mathrm{Im}A^{T}}\right|_{+}+O(L^{-1}).
\]
We have therefore build parameters $\theta$ that represent the function
$x\mapsto Ax$ with parameter norm $\left\Vert \theta\right\Vert ^{2}=L\mathrm{Rank}A+2\log\left|A\right|_{+}+O(L^{-1})$,
which upper bounds the representation cost, thus implying that $R^{(1)}(x\mapsto Ax;\Omega)\leq2\log\left|A\right|_{+}$
as needed.
\end{proof}

\subsection{Identity}
\begin{prop}[Proposition \ref{prop:R1_identity} from the main]
\label{prop:planar_R1=00003D0}For a domain with $\mathrm{Rank}_{J}(id;\Omega)=\mathrm{Rank}_{BN}(id;\Omega)=k$,
then $R^{(1)}(id;\Omega)=0$ if and only if $\Omega$ is $k$-planar
and completely positive.
\end{prop}

\begin{proof}
First if $\Omega$ is completely positive and $k$-planar one can
represent the identity with a depth $L$ network of parameter norm
$Lk$, by taking $W_{1}=O,W_{\ell}=P_{\mathrm{Im}O},W_{L}=O^{T}$
where $O$ is the $m\times d$ so that $O\Omega\subset\mathbb{R}_{+}^{m}$
and $O^{T}O=P_{\mathrm{span}\Omega}$. Thus $R^{(1)}(id;\Omega)=0$
(and all other corrections as well).

We will show that for any two points $x,y\in\Omega$ with $k$-dim
tangent spaces, their tangent spaces must match if $R^{(1)}(id;\Omega)=0$.

Let $A=J\alpha^{(L-1)}(x)_{|T_{x}\Omega}$ and $B=J\alpha^{(L-1)}(y)_{|T_{y}\Omega}$
be the be the Jacobian of the last hidden activations restricted to
the tangent spaces, we know that 
\begin{align*}
P_{T_{x}\Omega} & =W_{L}A\\
P_{T_{y}\Omega} & =W_{L}B
\end{align*}
so that given any weight matrix $W_{L}$ whose image contains $T_{x}\Omega$
and $T_{y}\Omega$, we can write 
\begin{align*}
A & =W_{L}^{+}P_{T_{x}\Omega}\\
B & =W_{L}^{+}P_{T_{y}\Omega}.
\end{align*}

Without loss of generality, we may assume that the span of $T_{x}\Omega$
and $T_{y}\Omega$ is full output space, and therefore that $W_{L}W_{L}^{T}$
is invertible.

Now we now that any parameters that represent the identity on $\Omega$
and has $A=J\alpha^{(L-1)}(x)_{|T_{x}\Omega}$ and $B=J\alpha^{(L-1)}(y)_{|T_{y}\Omega}$
must have parameter norm at least
\[
\left\Vert W_{L}\right\Vert _{F}^{2}+k(L-1)+\max\left\{ 2\log\left|A\right|_{+},2\log\left|B\right|_{+}\right\} .
\]

Subtracting $kL$ and taking $L\to\infty$, we obtain that
\[
R^{(1)}(id;\Omega)\geq\min_{W_{L}}\left\Vert W_{L}\right\Vert _{F}^{2}-k+\max\left\{ 2\log\left|W_{L}^{+}P_{T_{x}\Omega}\right|_{+},2\log\left|W_{L}^{+}P_{T_{y}\Omega}\right|_{+}\right\} .
\]

If we optimize $W_{L}$ only up to scaling (i.e. optimize $aW_{L}$
over $a$) we see that at the optimum, we always have $\left\Vert W_{L}\right\Vert _{F}^{2}=k$.
This allows us to rewrite the optimization as 
\[
R^{(1)}(id;\Omega)\geq\min_{\left\Vert W_{L}\right\Vert _{F}^{2}=k,}\max\left\{ 2\log\left|W_{L}^{+}P_{T_{x}\Omega}\right|_{+},2\log\left|W_{L}^{+}P_{T_{y}\Omega}\right|_{+}\right\} .
\]
The only way to put the first term inside the maximum to $0$ is to
have $W_{L}W_{L}^{T}=P_{T_{x}\Omega}$, but this leads to an exploding
second term if $P_{T_{x}\Omega}\neq P_{T_{y}\Omega}$.
\end{proof}
Under the assumption of $C$-uniform Lipschitzness of the representations
(that for all $\ell$, the functions $\tilde{\alpha}_{\ell}$ and
$(\alpha_{\ell}\to f_{\theta})$ are $C$-Lipschitz), one can show
a stronger version of the above:
\begin{prop}
\label{prop:uniform_Lipschitz_curvature}For a $C$-uniformly Lipschitz
sequence of ReLU networks representing the function $f$, we have
\[
R^{(1)}(f)\geq\log\left|Jf(x)\right|_{+}+\log\left|Jf(y)\right|_{+}+C^{-2}\left\Vert Jf_{\theta}(x)-Jf_{\theta}(y)\right\Vert _{*}.
\]
\end{prop}

\begin{proof}
The decomposition of the difference
\[
Jf_{\theta}(x)-Jf_{\theta}(y)=\sum_{\ell=1}^{L-1}W_{L}D_{L-1}(y)\cdots W_{\ell+1}\left(D_{\ell}(x)-D_{\ell}(y)\right)W_{\ell}D_{\ell-1}(x)\cdots D_{1}(x)W_{1},
\]
for the $w_{\ell}\times w_{\ell}$ diagonal matrices $D_{\ell}(x)=\mathrm{diag}(\dot{\sigma}(\tilde{\alpha}_{\ell}(x)))$,
implies the bound

\begin{align*}
\left\Vert Jf_{\theta}(x)-Jf_{\theta}(y)\right\Vert _{*} & \leq\sum_{\ell=1}^{L-1}\left\Vert W_{L}D_{L-1}(y)\cdots D_{\ell+1}(y)\right\Vert _{op}\left\Vert W_{\ell+1}\left(D_{\ell}(x)-D_{\ell}(y)\right)W_{\ell}\right\Vert _{*}\left\Vert D_{\ell-1}(x)\cdots D_{1}(x)W_{1}\right\Vert _{op}\\
 & \leq\frac{C^{2}}{2}\sum_{\ell=1}^{L-1}\left(\left\Vert W_{\ell+1}\left(D_{\ell}(x)-D_{\ell}(y)\right)\right\Vert _{F}^{2}+\left\Vert \left(D_{\ell}(x)-D_{\ell}(y)\right)W_{\ell}\right\Vert _{F}^{2}\right)
\end{align*}
since $\left\Vert AB\right\Vert _{*}\leq\frac{\left\Vert A\right\Vert _{F}^{2}+\left\Vert B\right\Vert _{F}^{2}}{2}$
and $\left(D_{\ell}(x)-D_{\ell}(y)\right)^{2}=\left(D_{\ell}(x)-D_{\ell}(y)\right)$.

Now since
\begin{align*}
L\left\Vert Jf_{\theta}(x)\right\Vert _{\nicefrac{2}{L}}^{\nicefrac{2}{L}} & \leq\sum_{\ell=1}^{L}\left\Vert W_{\ell}D_{\ell-1}(x)\right\Vert _{F}^{2}\\
L\left\Vert Jf_{\theta}(x)\right\Vert _{\nicefrac{2}{L}}^{\nicefrac{2}{L}} & \leq\sum_{\ell=1}^{L}\left\Vert D_{\ell}(x)W_{\ell}\right\Vert _{F}^{2}
\end{align*}
with the convention $D_{0}(x)=I_{d_{in}}$ and $D_{L}(x)=I_{d_{out}}$.
We obtain that 
\begin{align*}
L\left\Vert Jf_{\theta}(x)\right\Vert _{\nicefrac{2}{L}}^{\nicefrac{2}{L}}+L\left\Vert Jf_{\theta}(y)\right\Vert _{\nicefrac{2}{L}}^{\nicefrac{2}{L}} & \leq\frac{1}{2}\sum_{\ell=1}^{L}\left\Vert W_{\ell}D_{\ell-1}(x)\right\Vert _{F}^{2}+\left\Vert W_{\ell}D_{\ell-1}(y)\right\Vert _{F}^{2}+\left\Vert D_{\ell}(x)W_{\ell}\right\Vert _{F}^{2}+\left\Vert D_{\ell}(y)W_{\ell}\right\Vert _{F}^{2}\\
 & \leq\sum_{\ell=1}^{L}2\left\Vert W_{\ell}\right\Vert _{F}^{2}-\frac{1}{2}\left\Vert W_{\ell}\left(D_{\ell-1}(x)-D_{\ell-1}(y)\right)\right\Vert _{F}^{2}-\frac{1}{2}\left\Vert \left(D_{\ell}(x)-D_{\ell}(y)\right)W_{\ell}\right\Vert _{F}^{2}.
\end{align*}
This implies the bound 
\begin{align*}
\left\Vert \theta\right\Vert ^{2} & \geq\frac{L\left\Vert Jf_{\theta}(x)\right\Vert _{\nicefrac{2}{L}}^{\nicefrac{2}{L}}+L\left\Vert Jf_{\theta}(y)\right\Vert _{\nicefrac{2}{L}}^{\nicefrac{2}{L}}}{2}+C^{-2}\left\Vert Jf_{\theta}(x)-Jf_{\theta}(y)\right\Vert _{*}
\end{align*}
and thus 
\[
R^{(1)}(f)\geq\log\left|Jf(x)\right|_{+}+\log\left|Jf(y)\right|_{+}+C^{-2}\left\Vert Jf_{\theta}(x)-Jf_{\theta}(y)\right\Vert _{*}.
\]
\end{proof}

\subsection{Second Correction}
\begin{prop}[Proposition \ref{prop:second-correction-properties} from the main]
If there is a limiting representation as $L\to0$ in the optimal
representation of $f$, then $R^{(2)}(f)\geq0$. Furthermore:
\begin{enumerate}
\item If $R^{(0)}(f\circ g)=R^{(0)}(f)=R^{(0)}(g)$ and $R^{(1)}(f\circ g)=R^{(1)}(f)+R^{(1)}(g)$,
then $\sqrt{R^{(2)}(f\circ g)}\leq\sqrt{R^{(2)}(f)}+\sqrt{R^{(2)}(g)}$.
\item If $R^{(0)}(f+g)=R^{(0)}(f)+R^{(0)}(g)$ and $R^{(1)}(f+g)=R^{(1)}(f)+R^{(1)}(g)$,
then $R^{(2)}(f+g)\leq R^{(2)}(f)+R^{(2)}(g)$.
\item If $A^{p}\Omega$ is $k=\mathrm{Rank}A$-dimensional and completely
positive for all $p\in[0,1]$, where $A^{p}$ has its non-zero singular
taken to the $p$-th power, then $R^{(2)}(x\mapsto Ax;\Omega)=\frac{1}{2}\left\Vert \log_{+}A^{T}A\right\Vert ^{2}$.
\end{enumerate}
\end{prop}

\begin{proof}
We start from the inequality
\[
R(f\circ g;\Omega,L_{f}+L_{g})\leq R(f;g(\Omega),L_{f})+R(g;\Omega,L_{g}).
\]
We subtract $(L_{f}+L_{g})R^{(0)}(f\circ g)+R^{(1)}(f\circ g)$ divide
by $L_{f}+L_{g}$ and take the limit of increasing depths $L_{f},L_{g}$
with $\lim_{L_{g},L_{f}\to\infty}\frac{L_{f}}{L_{f}+L_{g}}=p\in(0,1)$
to obtain
\begin{equation}
R^{(2)}(f\circ g;\Omega)\leq\frac{1}{1-p}R^{(2)}(f;g(\Omega))+\frac{1}{p}R^{(2)}(g;\Omega).\label{eq:composition_R2}
\end{equation}

If $K_{p}$ is the limiting representation at a ratio $p\in(0,1)$,
we have $R^{(2)}(f;\Omega)=\frac{1}{p}R^{(2)}(K_{p};\Omega)+\frac{1}{1-p}R^{(2)}(K_{p}\to f;\Omega)$
and $p$ must minimize the RHS since if it was instead minimized at
a different ratio $p'\neq p$, one could find a lower norm representation
by mapping to $K_{p}$ in the first $p'L$ layers and then back to
the outputs. Now there are two possiblities, either $R^{(2)}(K_{p};\Omega)$
and $R^{(2)}(K_{p}\to f;\Omega)$ are non-negative in which case the
minimum is attained at some $p\in(0,1)$ and $R^{(2)}(f;\Omega)\geq0$,
or one or both is negative in which case the above is minimized at
$p\in\{0,1\}$ and $R^{(2)}(f;\Omega)=-\infty$. Since we assumed
$p\in(0,1)$, we are in the first case.

(1) To prove the first property, we optimize the RHS of \ref{eq:composition_R2}
over all possible choices of $p$ (and assuming that $R^{(2)}(f;g(\Omega)),R^{(2)}(g;\Omega)\geq0$)
we obtain
\[
\sqrt{R^{(2)}(f\circ g;\Omega)}\leq\sqrt{R^{(2)}(f;g(\Omega))}+\sqrt{R^{(2)}(g;\Omega)}.
\]

(2) This follows from the inequality $R(f+g;\Omega,L)\leq R(f;g(\Omega),L)+R(g;\Omega,L)$
after subtracting the $R^{(0)}$ and $R^{(1)}$ terms, dividing by
$L$ and taking $L\to\infty$.

(3) If $A=USV^{T}$, one chooses $W_{\ell}=U_{\ell}S^{\frac{1}{L}}U_{\ell-1}^{T}$
with $U_{0}=V$, $U_{L}=U$ and $U_{\ell}$ chosen so that $U_{\ell}S^{\frac{\ell}{L}}V^{T}\Omega\in\mathbb{R}_{+}^{n_{\ell}}$,
choosing large enough widths $n_{\ell}$. This choice of representation
of $A$ is optimal, i.e. its parameter norm matches the representation
cost $L\mathrm{Tr}\left[S^{\frac{2}{L}}\right]=L\mathrm{Rank}A+2\log\left|A\right|_{+}+\frac{1}{2L}\left\Vert \log_{+}A^{T}A\right\Vert ^{2}+O(L^{-2})$.

We know that
\[
\lim_{L\to\infty}R^{(1)}(\alpha_{\ell_{1}}\to\alpha_{\ell_{2}};\Omega)=R^{(1)}(f_{\theta};\Omega)\lim_{L\to\infty}\frac{\ell_{2}-\ell_{1}}{L}
\]
\[
\frac{1}{p}R^{(2)}(\alpha;\Omega)+\frac{1}{1-p}R^{(2)}(\alpha\to f;\Omega)\geq R^{(2)}(f;\Omega)
\]
\[
\frac{1}{p}R^{(2)}(\alpha;\Omega)+\frac{1}{1-p}R^{(2)}(\alpha\to f;\Omega)\geq R^{(2)}(f;\Omega)
\]
\end{proof}

\section{Bottleneck Structure}

This first result shows the existence of a Bottleneck structure on
the weight matrices:
\begin{thm}[Theorem \ref{thm:BN-structure-weights} from the main]
\label{thm:BN-structure-W_ell}Given parameters $\theta$ of a depth
$L$ network, with $\left\Vert \theta\right\Vert ^{2}\leq kL+c_{1}$
and a point $x$ such that $\mathrm{Rank}Jf_{\theta}(x)=k$, then
there are $w_{\ell}\times k$ (semi-)orthonormal $U_{\ell},V_{\ell}$
such that 
\[
\sum_{\ell=1}^{L}\left\Vert W_{\ell}-U_{\ell}V_{\ell}^{T}\right\Vert _{F}^{2}+\left\Vert b_{\ell}\right\Vert ^{2}\leq c_{1}-2\log\left|Jf_{\theta}(x)\right|_{+}
\]
 thus for any $p\in(0,1)$ there are at least $(1-p)L$ layers $\ell$
with
\begin{align*}
\left\Vert W_{\ell}-U_{\ell}V_{\ell}^{T}\right\Vert _{F}^{2}+\left\Vert b_{\ell}\right\Vert ^{2} & \leq\frac{c_{1}-2\log\left|Jf_{\theta}(x)\right|_{+}}{pL}.
\end{align*}
\end{thm}

\begin{proof}
Since 
\begin{align*}
Jf_{\theta}(x) & =W_{L}D_{L-1}(x)\cdots D_{1}(x)W_{1}\\
 & =W_{L}P_{\mathrm{Im}J\alpha_{L-1}(x)}D_{L-1}(x)\cdots P_{\mathrm{Im}J\alpha_{1}(x)}D_{1}(x)W_{1}
\end{align*}
Let us define the $k$-determinant $\left|A\right|_{k}=\prod_{i=1}^{k}s_{i}(A)$
to be the product of the $k$ largest singular values of a matrix
$A$. We have that $\left|AB\right|_{k}\leq\left|A\right|_{k}\left|B\right|_{k}$.

Since $Jf_{\theta}(x)$ has rank $k$, we have
\begin{align*}
\left|Jf_{\theta}(x)\right|_{+} & =\left|Jf_{\theta}(x)\right|_{k}\\
 & \leq\left|W_{L}\right|_{k}\left|D_{L-1}(x)\right|_{k}\left|W_{L-1}\right|_{k}\cdots\left|D_{1}(x)\right|_{k}\left|W_{1}\right|_{k}\\
 & =\left|W_{L}\right|_{k}\left|W_{L-1}\right|_{k}\cdots\left|W_{1}\right|_{k},
\end{align*}
where we used the fact that $\left|D_{\ell}(x)\right|_{k}=1$, because
the singular values of $D_{\ell}(x)$ are either $1$ or $0$, and
$1$ must have multiplicity at least $1$, otherwise $Jf_{\theta}(x)$
couldn't be rank $k$. 

This implies that 
\begin{align*}
 & \sum_{\ell=1}^{L}\left\Vert W_{\ell}\right\Vert _{F}^{2}+\left\Vert b_{\ell}\right\Vert _{F}^{2}-k-2\log\left|W_{\ell}\right|_{k}\leq\left\Vert \theta\right\Vert ^{2}-kL-2\log\left|Jf_{\theta}(x)\right|_{+}\leq c_{1}-2\log\left|Jf_{\theta}(x)\right|_{+}.
\end{align*}

Our gol is to show that the LHS is a sum of positive values which
sum up to a finite positive value, which will imply that most of the
summands must be very small. 

Given the SVD decomposition $W_{\ell}=U_{\ell}S_{\ell}V_{\ell}^{T}$,
we write $U_{\ell,k},S_{\ell,k},V_{\ell,k}$ for the submatrices corresponding
only to the $k$-largest singular values. We then have 
\begin{align*}
\left\Vert W_{\ell}-U_{\ell,k}V_{\ell,k}^{T}\right\Vert ^{2} & =\sum_{i=1}^{k}(s_{i}(W_{\ell})-1)^{2}+\sum_{i=k+1}^{\mathrm{Rank}W_{\ell}}s_{i}(W_{\ell})^{2}\\
 & \leq\sum_{i=1}^{k}s_{i}(W_{\ell})^{2}-1-2\log s_{i}(W_{\ell})+\sum_{i=k+1}^{\mathrm{Rank}W_{\ell}}s_{i}(W_{\ell})^{2}\\
 & =\left\Vert W_{\ell}\right\Vert ^{2}-k-2\log\left|W_{\ell}\right|_{k}
\end{align*}

We therefore have 
\[
\sum_{\ell=1}^{L}\left\Vert W_{\ell}-U_{\ell,k}V_{\ell,k}^{T}\right\Vert _{F}^{2}+\left\Vert b_{\ell}\right\Vert ^{2}\leq c_{1}-2\log\left|Jf_{\theta}(x)\right|_{+}.
\]
And for any $p\in(0,1)$ there are at most $pL$ layers $\ell$ with
\begin{align*}
\left\Vert W_{\ell}-U_{\ell,k}V_{\ell,k}^{T}\right\Vert _{F}^{2}+\left\Vert b_{\ell}\right\Vert ^{2} & \leq\frac{c_{1}-2\log\left|Jf_{\theta}(x)\right|_{+}}{pL}.
\end{align*}
\end{proof}
The fact that almost all weight matrices $W_{\ell}$ are approximately
$k$-dim would imply that the pre-activations $\tilde{\alpha}_{\ell}(X)=W_{\ell}\alpha_{\ell-1}(X)$
are $k$-dim too under the condition that the activations $\alpha_{\ell-1}(X)$
do not diverge. Assuming a bounded NTK is sufficient to guarantee
that these activations converge at almost every layer:
\begin{thm}[Theorem \ref{thm:NTK-implies-bounded-representations} from the main]
\label{thm:bounded-NTK-implies-bounded-repr}Given balanced parameters
$\theta$ of a depth $L$ network, with $\left\Vert \theta\right\Vert ^{2}\leq kL+c_{1}$
and a point $x$ such that $\mathrm{Rank}Jf_{\theta}(x)=k$ then if
$\frac{1}{N}\mathrm{Tr}\left[\Theta^{(L)}(x,x)\right]\leq cL$, then
$\sum_{\ell=1}^{L}\left\Vert \alpha_{\ell-1}(x)\right\Vert _{2}^{2}\leq\frac{c\max\{1,e^{\frac{c_{1}}{k}}\}}{k\left|Jf_{\theta}(x)\right|_{+}^{\nicefrac{2}{k}}}L$
and thus for all $p\in(0,1)$ there are at least $(1-p)L$ layers
such that 
\[
\left\Vert \alpha_{\ell-1}(x)\right\Vert _{2}^{2}\leq\frac{1}{p}\frac{c\max\{1,e^{\frac{c_{1}}{k}}\}}{k\left|Jf_{\theta}(x)\right|_{+}^{\nicefrac{2}{k}}}.
\]
\end{thm}

\begin{proof}
We have
\[
\mathrm{Tr}\left[\Theta^{(L)}(x,x)\right]=\sum_{\ell=1}^{L}\left\Vert \alpha_{\ell-1}(x)\right\Vert _{2}^{2}\left\Vert J(\tilde{\alpha}_{\ell}\to\alpha_{L})(x)\right\Vert _{F}^{2},
\]
we therefore need to lower bound $\left\Vert J(\tilde{\alpha}_{\ell}\to\alpha_{L})(x)\right\Vert _{F}^{2}$
to show that the activations $\left\Vert \alpha_{\ell-1}(x)\right\Vert _{2}^{2}$
must be bounded at almost every layer.

We will lower bound $\left\Vert J(\tilde{\alpha}_{\ell}\to\alpha_{L})(x)\right\Vert _{F}^{2}$
by $\left\Vert J(\tilde{\alpha}_{\ell}\to\alpha_{L})(x)P_{\ell}\right\Vert _{F}^{2}$
for $P_{\ell}$ the orthogonal projection to the image $\mathrm{Im}J\tilde{\alpha}_{\ell}(x)$.
Note $J(\tilde{\alpha}_{\ell}\to\alpha_{L})(x)P_{\ell}$ and $Jf_{\theta}(x)$
have the same rank.

By the arithmetic-geometric mean inequality, we have $\left\Vert A\right\Vert _{F}^{2}\geq\mathrm{Rank}A\left|A\right|_{+}^{\nicefrac{2}{k}}$,
yielding
\begin{align*}
\left\Vert J(\tilde{\alpha}_{\ell}\to\alpha_{L})(x)P_{\ell}\right\Vert _{F}^{2} & \geq k\left|J(\tilde{\alpha}_{\ell}\to\alpha_{L})(x)P_{\ell}\right|_{+}^{\nicefrac{2}{k}}.
\end{align*}

Now the balancedness of the parameters (i.e. $\left\Vert W_{\ell,i\cdot}\right\Vert ^{2}+b_{\ell,i}^{2}=\left\Vert W_{\ell+1,\cdot i}\right\Vert ^{2},\forall\ell,i$)
implies that the parameter norms are increasing $\left\Vert W_{\ell+1}\right\Vert _{F}^{2}\geq\left\Vert W_{\ell}\right\Vert _{F}^{2}$
and thus
\[
\frac{\left\Vert W_{\ell+1}\right\Vert ^{2}+\cdots+\left\Vert W_{L}\right\Vert ^{2}}{L-\ell}\geq\frac{\left\Vert W_{1}\right\Vert ^{2}+\cdots+\left\Vert W_{\ell}\right\Vert ^{2}}{\ell}.
\]
Thus
\begin{align*}
\frac{\left\Vert W_{1}\right\Vert ^{2}+\cdots+\left\Vert W_{\ell}\right\Vert ^{2}}{\ell} & =\frac{\left(\left\Vert W_{1}\right\Vert ^{2}+\cdots+\left\Vert W_{\ell}\right\Vert ^{2}\right)}{L}+\frac{L-\ell}{L}\frac{\left(\left\Vert W_{1}\right\Vert ^{2}+\cdots+\left\Vert W_{\ell}\right\Vert ^{2}\right)}{\ell}\\
 & \leq\frac{\left\Vert \theta\right\Vert ^{2}}{L}
\end{align*}
and 
\[
\left|J\tilde{\alpha}_{\ell}(x)\right|_{+}^{\nicefrac{2}{k\ell}}\leq\frac{1}{k}\left\Vert J\tilde{\alpha}_{\ell}(x)\right\Vert _{\nicefrac{2}{\ell}}^{\nicefrac{2}{\ell}}\leq\frac{\left\Vert W_{1}\right\Vert ^{2}+\cdots+\left\Vert W_{\ell}\right\Vert ^{2}}{k\ell}\leq\frac{\left\Vert \theta\right\Vert ^{2}}{kL}\leq1+\frac{c_{1}}{kL}
\]

and therefore
\begin{align*}
\left\Vert J(\tilde{\alpha}_{\ell}\to f_{\theta})(x)P_{\ell}\right\Vert _{F}^{2} & \geq k\left|J(\tilde{\alpha}_{\ell}\to f_{\theta})(x)P_{\ell}\right|_{+}^{\nicefrac{2}{k}}\\
 & =k\frac{\left|Jf_{\theta}(x)\right|_{+}^{\nicefrac{2}{k}}}{\left|J\tilde{\alpha}(x)\right|_{+}^{\nicefrac{2}{k}}}\\
 & \geq k\frac{\left|Jf_{\theta}(x)\right|_{+}^{\nicefrac{2}{k}}}{\left(1+\frac{c_{1}}{L}\right)^{\ell}}\\
 & \geq k\left|Jf_{\theta}(x)\right|_{+}^{\nicefrac{2}{k}}e^{-\frac{\ell}{L}\frac{c_{1}}{k}}\\
 & \geq k\left|Jf_{\theta}(x)\right|_{+}^{\nicefrac{2}{k}}\min\{1,e^{-\frac{c_{1}}{k}}\}.
\end{align*}
Thus 
\[
\sum_{\ell=1}^{L}\left\Vert \alpha_{\ell-1}(x)\right\Vert _{2}^{2}\leq\frac{c\max\{1,e^{\frac{c_{1}}{k}}\}}{k\left|Jf_{\theta}(x)\right|_{+}^{\nicefrac{2}{k}}}L
\]
which implies that there are at most $pL$ layers $\ell$ with 
\[
\left\Vert \alpha_{\ell-1}(x)\right\Vert _{2}^{2}\geq\frac{1}{p}\frac{c\max\{1,e^{\frac{c_{1}}{k}}\}}{k\left|Jf_{\theta}(x)\right|_{+}^{\nicefrac{2}{k}}}.
\]
\end{proof}
Note that for the MSE loss in the limit $\lambda\searrow0$, the Hessian
at a global minimum (i.e. $f_{\theta}$ interpolates the training
set) equals $\mathrm{Tr}\left[\mathcal{H}\mathcal{L}(\theta)\right]=\frac{1}{N}\mathrm{Tr}\left[\Theta^{(L)}(X,X)\right]$.
If we then assume that the trace of the Hessian is bounded by $cL$,
we get that $\mathrm{Tr}\left[\Theta^{(L)}(X,X)\right]\leq cNL$ and
thus there are at least $(1-p)L$ layers where 
\[
\frac{1}{N}\left\Vert \alpha_{\ell-1}(X)\right\Vert _{F}^{2}\leq\frac{1}{p}\frac{c\max\{1,e^{\frac{c_{1}}{k}}\}}{k\left|Jf_{\theta}(x)\right|_{+}^{\nicefrac{2}{k}}},
\]
thus guaranteeing the infinite depth convergence of training set activations
$\alpha_{\ell-1}(X)$ on those layers.

Putting the two above theorems together, we can prove that the pre-activations
are $k$-dim at almost every layer:
\begin{cor}[Corollary \ref{cor:BN-structure-alpha} from the main]
Given balanced parameters $\theta$ of a depth $L$ network with
$\left\Vert \theta\right\Vert ^{2}\leq kL+c_{1}$ and a set of points
$x_{1},\dots,x_{N}$ such that $\mathrm{Rank}Jf_{\theta}(x_{i})=k$
and $\frac{1}{N}\mathrm{Tr}\left[\Theta^{(L)}(X,X)\right]\leq cL$,
then for all $p\in(0,1)$ there are at least $(1-p)L$ layers such
that 
\begin{align*}
s_{k+1}\left(\frac{1}{\sqrt{N}}\tilde{\alpha}_{\ell}(X)\right) & \leq\sqrt{c_{1}-2\log\left|Jf_{\theta}(x)\right|_{+}}\left(\sqrt{\frac{1}{N}\sum_{i=1}^{N}\frac{c\max\{1,e^{\frac{c_{1}}{k}}\}}{k\left|Jf_{\theta}(x_{i})\right|_{+}^{\nicefrac{2}{k}}}}+\sqrt{p}\right)\frac{1}{p\sqrt{L}}
\end{align*}
\end{cor}

\begin{proof}
Since $\tilde{\alpha}_{\ell}(X)=W_{\ell}\alpha_{\ell-1}(X)+b_{\ell}\mathbf{1}_{N}^{T}$
we know that
\[
s_{k+1}\left(\frac{1}{\sqrt{N}}\tilde{\alpha}_{\ell}(X)\right)\leq s_{k+1}(W_{\ell})\left\Vert \frac{1}{\sqrt{N}}\alpha_{\ell-1}(X)\right\Vert _{op}+\left\Vert b_{\ell}\right\Vert .
\]
By Theorems \ref{thm:BN-structure-W_ell} and \ref{thm:bounded-NTK-implies-bounded-repr},
there are for any $p\in(0,\frac{1}{2})$ at least $(1-2p)L$ layers
such that 
\begin{align*}
\left\Vert W_{\ell}-V_{\ell}V_{\ell-1}^{T}\right\Vert _{F}^{2}+\left\Vert b_{\ell}\right\Vert ^{2} & \leq\frac{c_{1}-2\log\left|Jf_{\theta}(x)\right|_{+}}{pL}\\
\left\Vert \frac{1}{\sqrt{N}}\alpha_{\ell-1}(X)\right\Vert _{F}^{2} & \leq\sum_{i=1}^{N}\frac{1}{p}\frac{c\max\{1,e^{\frac{c_{1}}{k}}\}}{k\left|Jf_{\theta}(x_{i})\right|_{+}^{\nicefrac{2}{k}}}
\end{align*}
so that
\[
s_{k+1}\left(\frac{1}{\sqrt{N}}\tilde{\alpha}_{\ell}(X)\right)\leq\sqrt{\frac{c_{1}-2\log\left|Jf_{\theta}(x)\right|_{+}}{p}}\left(\sqrt{\frac{1}{pN}\sum_{i=1}^{N}\frac{c\max\{1,e^{\frac{c_{1}}{k}}\}}{k\left|Jf_{\theta}(x_{i})\right|_{+}^{\nicefrac{2}{k}}}}+1\right)\frac{1}{\sqrt{L}}.
\]
\end{proof}

\section{Local Minima Stability}

In this section we motivate the assumption that the Jacobian $J\tilde{\alpha}_{\ell}(x)$
is uniformly bounded in operator norm as $L\to\infty$. The idea is
that solutions with a blowing up Jacobian $J\tilde{\alpha}_{\ell}(x)$
correspond to very narrow local minima.

The narrowness of a local minimum is related to the Neural Tangent
Kernel (or Fisher matrix). We have that 
\[
\mathrm{Tr}\left[\Theta^{(L)}(x,x)\right]=\sum_{\ell=1}^{L}\left\Vert \alpha_{\ell-1}(x)\right\Vert ^{2}\left\Vert J(\tilde{\alpha}_{\ell}\to f_{\theta})(x)\right\Vert _{F}^{2}.
\]
A large Jacobian $Jf_{\theta}(x)$ leads to a blow up of the derivative
of the NTK:
\begin{prop}[Proposition \ref{prop:blowup-NTK-Jacobian} from the main]
For any point $x$, we have
\[
\left\Vert \partial_{xy}^{2}\Theta(x,x)\right\Vert _{op}\geq2L\left\Vert Jf_{\theta}(x)\right\Vert _{op}^{2-\nicefrac{2}{L}}
\]
where $\partial_{xy}^{2}\Theta(x,x)$ is understood as a $d_{in}d_{out}\times d_{in}d_{out}$
matrix.

Furthermore, for any two points $x,y$ such that the pre-activations
of all neurons of the network remain constant on the segment $[x,y]$,
then either $\left\Vert \Theta(x,x)\right\Vert _{op}$ or $\left\Vert \Theta(y,y)\right\Vert _{op}$
is lower bounded by $\frac{L}{4}\left\Vert x-y\right\Vert ^{2}\left\Vert Jf_{\theta}(x)\frac{y-x}{\left\Vert x-y\right\Vert }\right\Vert _{2}^{2-\nicefrac{2}{L}}.$
\end{prop}

\begin{proof}
(1) For any point $x$, we have
\begin{align*}
\partial_{x,y}\left(v^{T}\Theta(x,x)v\right)[u,u] & =\sum_{\ell=1}^{L}u^{T}W_{1}^{T}D_{1}(x)\cdots D_{\ell-1}(x)^{2}\cdots D_{1}(x)W_{1}uv^{T}W_{L}D_{L-1}(x)\cdots D_{\ell}(x)^{2}\cdots D_{L-1}(x)W_{L}^{T}v\\
 & =\sum_{\ell=1}^{L}\left\Vert D_{\ell-1}(x)\cdots D_{1}(x)W_{1}u\right\Vert _{2}^{2}\left\Vert D_{\ell}(x)\cdots D_{L-1}(x)W_{L}v\right\Vert _{2}^{2}.
\end{align*}
On the other hand, we have
\begin{align*}
\left|v^{T}Jf_{\theta}(x)u\right| & =\left|v^{T}W_{L}D_{L-1}(x)\cdots D_{1}(x)W_{1}u\right|\\
 & \leq\left\Vert D_{\ell}(x)\cdots D_{1}(x)W_{1}u\right\Vert _{2}\left\Vert D_{\ell}(x)\cdots D_{L-1}(x)W_{L}v\right\Vert _{2},
\end{align*}
where we used the fact that $D_{\ell}(x)D_{\ell}(x)=D_{\ell}(x)$.
This applies to the case $\ell=L$ and $\ell=1$ too, using the definition
$D_{L}(x)=I_{d_{out}}$ and $D_{0}(x)=I_{d_{in}}$. This implies
\begin{align*}
\partial_{xy}^{2}\left(v^{T}\Theta(x,x)v\right)[u,u] & \geq\left|v^{T}Jf_{\theta}(x)u\right|^{2}\sum_{\ell=1}^{L}\frac{\left\Vert D_{\ell-1}(x)\cdots D_{1}(x)W_{1}u\right\Vert _{2}^{2}}{\left\Vert D_{\ell}(x)\cdots D_{1}(x)W_{1}u\right\Vert _{2}^{2}}\\
 & \geq\left|v^{T}Jf_{\theta}(x)u\right|^{2}L\left(\frac{\left\Vert u\right\Vert _{2}^{2}}{\left\Vert W_{L}D_{L-1}(x)\cdots D_{1}(x)W_{1}u\right\Vert _{2}^{2}}\right)^{\frac{1}{L}}\\
 & \geq L\frac{\left|v^{T}Jf_{\theta}(x)u\right|^{2}}{\left\Vert Jf_{\theta}(x)u\right\Vert _{2}^{\nicefrac{2}{L}}}.
\end{align*}
where we used the geometric/arithmetic mean inequality for the second
inequality.

If $u,v$ are right and left singular vectors of $Jf_{\theta}(x)$
with singular value $s$, then the above bound equals $Ls^{2-\frac{2}{L}}$.

(2) Now let us consider a segment $\gamma(t)=(1-t)x+ty$ between two
points $x,y$ with no changes of activations on these paths (i.e.
$D_{\ell}(\gamma(t))$ is constant for all $t\in[0,1]$). Defining
$u=\frac{y-x}{\left\Vert y-x\right\Vert }$ and $v=\frac{Jf_{\theta}(x)u}{\left\Vert Jf_{\theta}(x)u\right\Vert }$,
we have
\[
\partial_{t}v^{T}\Theta(\gamma(t),\gamma(t))v=\left\Vert x-y\right\Vert \partial_{x}\left(v^{T}\Theta(\gamma(t),\gamma(t))v\right)[u]+\left\Vert x-y\right\Vert \partial_{y}\left(v^{T}\Theta(\gamma(t),\gamma(t))v\right)[u]
\]
and since $\partial_{xx}\Theta(\gamma(t),\gamma(t))=0$ and $\partial_{yy}\Theta(\gamma(t),\gamma(t))=0$
for all $t\in[0,1]$, we have
\[
\partial_{t}^{2}\left(v^{T}\Theta(\gamma(t),\gamma(t))v\right)=2\left\Vert x-y\right\Vert ^{2}\partial_{xy}^{2}\left(v^{T}\Theta(\gamma(t),\gamma(t))v\right)[u,u]\geq2L\left\Vert x-y\right\Vert ^{2}\left\Vert Jf_{\theta}(x)u\right\Vert _{2}^{2-\nicefrac{2}{L}}.
\]
Since $v^{T}\Theta(\gamma(t),\gamma(t))v\geq0$ for all $t\in[0,1]$
then either
\[
v^{T}\Theta(x,x)v\geq\frac{L}{4}\left\Vert x-y\right\Vert ^{2}\left\Vert Jf_{\theta}(x)u\right\Vert _{2}^{2-\nicefrac{2}{L}}
\]
or 
\[
v^{T}\Theta(y,y)v\geq\frac{L}{4}\left\Vert x-y\right\Vert ^{2}\left\Vert Jf_{\theta}(x)u\right\Vert _{2}^{2-\nicefrac{2}{L}}.
\]
\end{proof}
Rank-underestimating fitting functions typically feature exploding
derivatives, which was used to show in \cite{jacot_2022_BN_rank}
that BN-rank 1 fitting functions must have a $R^{(1)}$ term that
blows up iwith the number of datapoints $N$. With some additional
work, we can show that the NTK will blow up at some $x$:
\begin{thm}[Theorem \ref{thm:BN-rank-1-fitting-NTK} from the main]
Let $f^{*}:\Omega\to\mathbb{R}^{d_{out}}$ be a function with Jacobian
rank $k^{*}>1$ (i.e. there is a $x\in\Omega$ with $\mathrm{Rank}Jf^{*}(x)=k^{*}$),
then with high probability over the sampling of a training set $x_{1},\dots,x_{N}$
(sampled from a distribution with support $\Omega$), we have that
for any parameters $\theta$ of a deep enough network that represent
a BN-rank 1 function $f_{\theta}$ that fits the training set $f_{\theta}(x_{i})=f^{*}(x_{i})$
with norm $\left\Vert \theta\right\Vert ^{2}=L+c_{1}$ then there
is a point $x\in\Omega$ where the NTK satisfies
\[
\Theta^{(L)}(x,x)\geq c''Le^{-c_{1}}N^{4-\frac{4}{k^{*}}}.
\]
\end{thm}

\begin{proof}
For all $i$ we define $d_{1,i}$ and $d_{2,i}$ to be the distance
between $y_{i}$ and its closest and second closest point in $y_{1},\dots,y_{N}$.
Following the argument in \cite{beardwood_1959_shortest_path_many_points},
the shortest path that goes through all points must be at least $\sum_{i=1}^{N}\frac{d_{1,i}+d_{2,i}}{2}$
(which would be tight if it is possible to always jump to the closest
or second closest point along the path). Since the expected distances
$d_{1,i}$ and $d_{2,i}$ are $N^{-\frac{1}{k^{*}}}$ since the $y_{i}$
are sampled from a $k^{*}$-dimensional distribution, the expected
length of the shortest path is of order $N^{1-\frac{1}{k^{*}}}$.
Actually most of the distance $d_{1,i}$ and $d_{2,i}$ will be of
order $N^{-\frac{1}{k^{*}}}$ with only a few outliers with larger
or smaller distances, thus for any subset of indices $I\subset[1,\dots,N]$
that contains a finite ratio of all indices, the sum $\sum_{i\in I}\frac{d_{1,i}+d_{2,i}}{2}$
will be of order $N^{1-\frac{1}{k^{*}}}$ too.

Following the argument in the proof Theorem 2 from \cite{jacot_2022_BN_rank},
we can reorder the indices so that the segment $[x_{1},x_{N}]$ will
mapped \textbf{$f_{\theta}$} to a path that goes through $y_{1},\dots,y_{N}$.
We can therefore define the points $\tilde{x}_{1},\dots,\tilde{x}_{N}$
that are preimages of $y_{1},\dots,y_{N}$ on the segment.

On the interval $[\tilde{x}_{i},\tilde{x}_{i+1}]$ there must a point
$x$ with $\left\Vert Jf_{\theta}(x)\right\Vert _{op}\geq\frac{\left\Vert y_{i+1}-y_{i}\right\Vert }{\left\Vert \tilde{x}_{i+1}-\tilde{x}_{i}\right\Vert }$.
Now since $\sum\left\Vert \tilde{x}_{i+1}-\tilde{x}_{i}\right\Vert =\left\Vert x_{N}-x_{1}\right\Vert \leq\mathrm{diag}\Omega$
there must be at least $p(N-1)$ intervals $i$ with $\left\Vert \tilde{x}_{i+1}-\tilde{x}_{i}\right\Vert \leq\frac{\mathrm{diag}\Omega}{(1-p)(N-1)}$,
and amongst those $i$s they would all satisfy $\left\Vert y_{i+1}-y_{i}\right\Vert \geq cN^{-\frac{1}{k^{*}}}$
except for a few outliers. Thus we can for example guarantee that
there are at least $\frac{5}{6}(N-1)$ intervals $[\tilde{x}_{i},\tilde{x}_{i+1}]$
that contain a point $x$ with $\left\Vert Jf_{\theta}(x)\right\Vert _{op}\geq cN^{1-\frac{1}{k^{*}}}$.

First observe that by Theorem 2 of \cite{jacot_2022_BN_rank}, there
must be a point $x\in\Omega$ such that $\left\Vert Jf_{\theta}(x)\right\Vert _{op}\geq N^{1-\frac{1}{k}}$
thus by Theorem \ref{thm:BN-structure-weights}, there are at least
$pL$ layers where 
\[
\left\Vert W_{\ell}-u_{\ell}v_{\ell}^{T}\right\Vert ^{2}\leq\frac{c_{1}-2\log\left|Jf_{\theta}(x)\right|_{+}}{pL}.
\]

Consider one of those $pL$ layers $\ell$, with activatons $z_{i}=\alpha_{\ell-1}(\tilde{x}_{i})$.
Let $i_{1},\dots,i_{N}$ be the ordering of the indices so that $u_{\ell}^{T}z_{i_{m}}$
is increasing in $m$. Then the hidden representations must satisfy
\[
\frac{\left\Vert W_{\ell}(z_{i_{m}}-z_{i_{m-1}})\right\Vert +\left\Vert W_{\ell}(z_{i_{m}}-z_{i_{m+1}})\right\Vert }{2}\geq e^{-\frac{\left\Vert \theta^{(\ell+1:L)}\right\Vert ^{2}-(L-\ell)}{2}}\frac{d_{1,i_{m}}+d_{2,i_{m}}}{2}
\]
and
\begin{align*}
\left\Vert W_{\ell}(z_{i_{m}}-z_{i_{m-1}})\right\Vert  & \leq u_{\ell}^{T}(z_{i_{m}}-z_{i_{m-1}})+\sqrt{\frac{c_{1}-2\log\left|Jf_{\theta}(x)\right|_{+}}{pL}}\left\Vert z_{i_{m}}-z_{i_{m-1}}\right\Vert .
\end{align*}

For any two indices $m_{1}<m_{2}$ separated by $pN$ indices (where
$p>0$ remains finite), we have

\begin{align*}
u_{\ell}^{T}(z_{i_{m_{2}}}-z_{i_{m_{1}}}) & \geq\sum_{m=m_{1}+1}^{m_{2}}\frac{\left\Vert W_{\ell}(z_{i_{m}}-z_{i_{m-1}})\right\Vert +\left\Vert W_{\ell}(z_{i_{m}}-z_{i_{m+1}})\right\Vert }{2}\\
 & -\sqrt{\frac{c_{1}-2\log\left|Jf_{\theta}(x)\right|_{+}}{pL}}\sum_{m=m_{1}+1}^{m_{2}}\frac{\left\Vert z_{i_{m}}-z_{i_{m-1}}\right\Vert +\left\Vert z_{i_{m}}-z_{i_{m+1}}\right\Vert }{2}\\
 & \geq e^{-\frac{\left\Vert \theta^{(\ell+1:L)}\right\Vert ^{2}-(L-\ell)}{2}}\sum_{m=m_{1}+1}^{m_{2}}\frac{d_{1,i_{m}}+d_{2,i_{m}}}{2}\\
 & -(m_{2}-m_{1})e^{\frac{\left\Vert \theta^{(1:\ell-1)}\right\Vert ^{2}-(\ell-1)}{2}}\sqrt{\frac{c_{1}-2\log\left|Jf_{\theta}(x)\right|_{+}}{pL}}\mathrm{diam}\Omega\\
 & \geq(m_{2}-m_{1})\left[ce^{-\frac{\left\Vert \theta^{(\ell+1:L)}\right\Vert ^{2}-(L-\ell)}{2}}N^{-\frac{1}{k^{*}}}-e^{\frac{\left\Vert \theta^{(1:\ell-1)}\right\Vert ^{2}-(\ell-1)}{2}}\sqrt{\frac{c_{1}-2\log\left|Jf_{\theta}(x)\right|_{+}}{pL}}\mathrm{diam}\Omega\right],
\end{align*}
where we used the fact that up to a few outliers $\frac{d_{1,i}+d_{2,i}}{2}=\Omega(N^{\frac{1}{k^{*}}})$.

Thus for $L\geq\frac{4}{c^{2}}e^{c_{1}+1}N^{\frac{2}{k^{*}}}\frac{c_{1}-2\log\left|Jf_{\theta}(x)\right|_{+}}{p}\left(\mathrm{diam}\Omega\right)^{2}$,
we have $u_{\ell}^{T}(z_{i_{m_{2}}}-z_{i_{m_{1}}})\geq(m_{2}-m_{1})\frac{c}{2}e^{-\frac{\left\Vert \theta^{(\ell+1:L)}\right\Vert ^{2}-(L-\ell)}{2}}N^{-\frac{1}{k^{*}}}$.
which implies that at least half of the activations $z_{i}$ have
norm larger than $\frac{c}{8}e^{-\frac{\left\Vert \theta^{(\ell+1:L)}\right\Vert ^{2}-(L-\ell)}{2}}N^{1-\frac{1}{k^{*}}}$.

This implies that at least one fourth of the indices $i$ satisfy
for at least one fourth of the $pL$ layers $\ell$ 
\[
\left\Vert \alpha_{\ell-1}(\tilde{x}_{i})\right\Vert \geq\frac{c}{8}e^{-\frac{\left\Vert \theta^{(\ell+1:L)}\right\Vert ^{2}-(L-\ell)}{2}}N^{1-\frac{1}{k^{*}}}.
\]

Now amongst these indices there are at least some such that there
is a point $x$ in the interval $[\tilde{x}_{i},\tilde{x}_{i+1}]$
with $\left\Vert Jf_{\theta}(x)\right\Vert _{op}\geq cN^{1-\frac{1}{k^{*}}}$.
Since $x$ is $O(N^{-1})$-close to $\tilde{x}_{i}$ one can guarantee
that $\left\Vert \alpha_{\ell-1}(x)\right\Vert \geq c'e^{-\frac{\left\Vert \theta^{(\ell+1:L)}\right\Vert ^{2}-(L-\ell)}{2}}N^{1-\frac{1}{k^{*}}}$
for some constant $c'$. 

But the Jacobian $J(\tilde{\alpha}_{\ell}\to f_{\theta})$ also explodes
at $x$ since
\[
\left\Vert J(\tilde{\alpha}_{\ell}\to f_{\theta})(x)\right\Vert _{op}\geq\frac{\left\Vert Jf_{\theta}(x)\right\Vert _{op}}{\left\Vert J\tilde{\alpha}_{\ell}(x)\right\Vert _{op}}\geq e^{-\frac{\left\Vert \theta^{(1:\ell)}\right\Vert ^{2}-\ell}{2}}cN^{1-\frac{1}{k^{*}}}.
\]

We can now lower bound the NTK 
\begin{align*}
\mathrm{Tr}\left[\Theta^{(L)}(x,x)\right] & =\sum_{\ell=1}^{L}\left\Vert \alpha_{\ell-1}(x)\right\Vert ^{2}\left\Vert J(\tilde{\alpha}_{\ell}\to f_{\theta})(x_{1})\right\Vert _{F}^{2}\\
 & \geq\frac{pL}{4}c'^{2}e^{-\left\Vert \theta^{(\ell+1:L)}\right\Vert ^{2}+(L-\ell)}N^{2-\frac{2}{k^{*}}}e^{-\left\Vert \theta^{(1:\ell)}\right\Vert ^{2}+\ell}c^{2}N^{2-\frac{2}{k^{*}}}\\
 & =c''Le^{-c_{1}}N^{4-\frac{4}{k^{*}}}.
\end{align*}
\end{proof}
This suggests that such points are avoided not only because they have
a large $R^{(1)}$ value, but also (if not mostly) because they lie
at the bottom of a very narrow valley.

\section{Technical Results}

\subsection{Regularity Counterexample}

We guve here an example of a simple function whose optimal representation
geodesic does not converge, due to it being not uniformly Lipschitz:
\begin{example}
The function $f:\Omega\to\mathbb{R}^{3}$ with $\Omega=[0,1]^{3}$
defined by 
\[
f(x,y,z)=\begin{cases}
(x,y,z) & \text{if }x\leq y\\
(x,y,z+a(x-y)) & \text{if }x>y
\end{cases}
\]
satisfies $R^{(0)}(f;\Omega)=3$ and $R^{(1)}(f;\Omega)=0$. The optimal
representations of $f$ are not uniformly Lipschitz as $L\to\infty$. 
\end{example}

\begin{proof}
While we are not able to identify exactly the optimal representation
geodesic for the function $f$, we will first show that $R^{(1)}(f;\Omega)=0$,
and then show that the uniform Lipschitzness of the optimal representations
would contradict with Proposition \ref{prop:uniform_Lipschitz_curvature}.

(1) Since the Jacobian takes two values inside $\mathbb{R}_{+}^{3}$,
either the identity $I_{3}$ or $\left(\begin{array}{ccc}
1 & 0 & 0\\
0 & 1 & 0\\
1 & -1 & 1
\end{array}\right)$, we know by Theorem \ref{thm:properties_first_correction-1} that
$R^{(1)}(f;\Omega)\geq2\log\left|I_{3}\right|_{+}=0$. We therefore
only need to construct a sequence of parameters of different depths
that represent $f$ with a squared parameter norm of order $3L+o(1)$.
For simplicity, we only do this construction for even depths (the
odd case can be constructed similarly). We define:

\begin{align*}
W_{\ell} & =\left(\begin{array}{ccc}
e^{\epsilon} & 0 & 0\\
0 & e^{\epsilon} & 0\\
0 & 0 & e^{-2\epsilon}
\end{array}\right)\text{for \ensuremath{\ell=1,\dots,\frac{L}{2}-1}}\\
W_{\frac{L}{2}} & =\left(\begin{array}{ccc}
1 & 0 & 0\\
0 & 1 & 0\\
0 & 0 & 1\\
\sqrt{a}e^{-\frac{L-2}{2}\epsilon} & -\sqrt{a}e^{-\frac{L-2}{2}\epsilon} & 0
\end{array}\right)\text{}\\
W_{\frac{L}{2}+1} & =\left(\begin{array}{cccc}
1 & 0 & 0 & 0\\
0 & 1 & 0 & 0\\
0 & 0 & 1 & \sqrt{a}e^{-(L-2)\epsilon}
\end{array}\right)\text{}\\
W_{\ell} & =\left(\begin{array}{ccc}
e^{-\epsilon} & 0 & 0\\
0 & e^{-\epsilon} & 0\\
0 & 0 & e^{2\epsilon}
\end{array}\right)\text{for \ensuremath{\ell=\frac{L}{2}+2,\dots,L}}
\end{align*}
We have for all $x\in\mathbb{R}_{+}^{3}$
\[
\alpha_{\frac{L}{2}-1}(x)=\left(\begin{array}{c}
e^{\frac{L-2}{2}\epsilon}x_{1}\\
e^{\frac{L-2}{2}\epsilon}x_{2}\\
e^{-(L-2)\epsilon}x_{3}
\end{array}\right)
\]
and 
\[
\alpha_{\frac{L}{2}}(x)=\left(\begin{array}{c}
e^{\frac{L-2}{2}\epsilon}x_{1}\\
e^{\frac{L-2}{2}\epsilon}x_{2}\\
e^{-(L-2)\epsilon}x_{3}\\
\sigma(x_{1}-x_{2})
\end{array}\right)
\]
and
\[
\alpha_{\frac{L}{2}+1}(x)=\left(\begin{array}{c}
e^{\frac{L-2}{2}\epsilon}x_{1}\\
e^{\frac{L-2}{2}\epsilon}x_{2}\\
e^{-(L-2)\epsilon}\left(x_{3}+\sigma(x_{1}-x_{2})\right)
\end{array}\right)
\]
and
\[
f_{\theta}(x)=\left(\begin{array}{c}
x_{1}\\
x_{2}\\
x_{3}+\sigma(x_{1}-x_{2})
\end{array}\right).
\]
The norm of the parameters is
\begin{align*}
 & \frac{L-2}{2}(2e^{2\epsilon}+e^{-4\epsilon})+(3+2e^{-(L-2)\epsilon})+(3+e^{-2(L-2)\epsilon})+\frac{L-2}{2}(2e^{-2\epsilon}+e^{4\epsilon})\\
 & =3L+2\left(e^{2\epsilon}-1\right)+(e^{-4\epsilon}-1)+2e^{-(L-2)\epsilon}+e^{-2(L-2)\epsilon}+2\left(e^{-2\epsilon}-1\right)+(e^{4\epsilon}-1)
\end{align*}
If we take $\epsilon=L^{-\gamma}$ for $\gamma\in(\frac{1}{2},1)$,
then the terms $2e^{-(L-2)\epsilon}$ and $e^{-2(L-2)\epsilon}$ decay
exponentially (at a rate of $e^{L^{1-\gamma}}$), in addition the
terms $2\left(e^{2\epsilon}-1\right)+(e^{-4\epsilon}-1)$ and $2\left(e^{-2\epsilon}-1\right)+(e^{4\epsilon}-1)$
are of order $L^{-2\gamma}$. This proves that $R^{(1)}(f;\Omega)=0$.

(2) Let us now assume that the optimal representation of $f$ is $C$-uniform
Lipschitz for some constant $C$, then by Proposition \ref{prop:uniform_Lipschitz_curvature},
we have that
\[
R^{(1)}(f;\Omega)\geq\log\left|I_{3}\right|_{+}+\log\left|\left(\begin{array}{ccc}
1 & 0 & 0\\
0 & 1 & 0\\
1 & -1 & 1
\end{array}\right)\right|_{+}+C^{-2}\left\Vert I_{3}-\left(\begin{array}{ccc}
1 & 0 & 0\\
0 & 1 & 0\\
1 & -1 & 1
\end{array}\right)\right\Vert _{*}>0,
\]
which contradicts with the fact that $R^{(1)}(f;\Omega)=0$.
\end{proof}

\subsection{Extension outside FPLFs}

Since all functions represented by finite depth and width networks
are FPLFs, the representation cost of any such function is infinite.
But we can define the representation cost of a function $f$ that
is the limit of a sequence of FPLF as the infimum over all sequences
$f_{i}\to f$ converging of $\lim_{i\to\infty}R(f_{i};\Omega)$ (for
some choice of convergence type that implies convergence of the Jacobians
$Jf_{i}(x)\to Jf(x)$). Note that since the representation cost $R(f;\Omega)$
is lower semi-continuous, i.e. $\lim\inf_{f\to f_{0}}R(f;\Omega)\geq R(f_{0};\Omega)$,
this does not change the definition of the representation cost on
the space of FPLFs.

\section{Numerical Experiments}

For the first numerical experiment, the data pairs $(x,y)$ were generated
as follows. First we sample a $8$-dimensional `latent vector' $z$,
from which we define $x=g(z_{1},\dots,z_{8})\in\mathbb{R}^{20}$ and
$y=h(z_{1},z_{2})\in\mathbb{R}^{20}$ for two random functions $g:\mathbb{R}^{8}\to\mathbb{R}^{20}$
and $h:\mathbb{R}^{2}\to\mathbb{R}^{20}$ given by two shallow networks
with random parameters. Assuming that $g$ is injective (which it
is with high probability), the function $f^{*}=h\circ g^{-1}$ which
maps $x$ to $y$ has BN-rank 2.

\end{document}